\documentclass{article}

\usepackage[preprint, nonatbib]{neurips_2024}

\usepackage[utf8]{inputenc} %
\usepackage[T1]{fontenc}    %
\usepackage[hidelinks,colorlinks,linkcolor={blue!70!black},citecolor={blue!70!black},urlcolor={blue!70!black}]{hyperref}       %
\usepackage{url}            %
\usepackage{booktabs}       %
\usepackage{amsfonts}       %
\usepackage{nicefrac}       %
\usepackage{microtype}      %
\usepackage{xcolor}         %
\usepackage{amsmath}
\usepackage{amssymb}
\usepackage{mathtools}
\usepackage{amsthm}
\usepackage{tikz}
\usepackage{subfigure}
\usepackage{pgfplots}
\usepackage{multirow}
\usepackage{algorithmic}
\usepackage{algorithm}
\usepackage{xspace}
\usepackage{xhfill}
\usepackage[backend=bibtex, style=ieee]{biblatex}
\usepgfplotslibrary{fillbetween}
\pgfplotsset{compat=1.17} %
\usepackage[capitalize,noabbrev]{cleveref}
\usepackage{wrapfig}
\theoremstyle{plain}
\newtheorem{theorem}{Theorem}[section]

\newtheorem{lemma}[theorem]{Lemma}

\theoremstyle{definition}
\newtheorem{definition}[theorem]{Definition}
\newtheorem{assumption}[theorem]{Assumption}
\theoremstyle{remark}

\newcounter{numrel}%
\counterwithin*{numrel}{section}%
\newcommand{\numrel}[1]{%
  \stepcounter{numrel}%
  \overset{(\roman{numrel})}{#1}%
}
\newcommand{\restartnumrel}{\setcounter{numrel}{0}}

\makeatletter
\newenvironment{breakablealgorithm}
  {%
   \begin{center}
     \refstepcounter{algorithm}%
     \hrule height.8pt depth0pt \kern2pt%
     \renewcommand{\caption}[2][\relax]{%
       {\raggedright\textbf{\ALG@name~\thealgorithm} ##2\par}%
       \ifx\relax##1\relax %
         \addcontentsline{loa}{algorithm}{\protect\numberline{\thealgorithm}##2}%
       \else %
         \addcontentsline{loa}{algorithm}{\protect\numberline{\thealgorithm}##1}%
       \fi
       \kern2pt\hrule\kern2pt
     }
  }{%
     \kern2pt\hrule\relax%
   \end{center}
  }
\makeatother

\title{Communication-Efficient Byzantine-Resilient Federated Zero-Order
Optimization}

\author{%
    Afonso de Sá Delgado Neto\\
  Technical University of Munich\\
  Munich, Germany\\
  \texttt{afonso.delgado@tum.de} \\
  \And
  Maximilian Egger \\
  Technical University of Munich\\
  Munich, Germany\\
    \texttt{maximilian.egger@tum.de} \\
  \AND
  Mayank Bakshi \\
  Arizona State University\\
  Tucson, AZ, United States \\
  \texttt{bakshi@asu.edu} \\
  \And
  Rawad Bitar \\
  Technical University of Munich\\
  Munich, Germany\\
    \texttt{rawad.bitar@tum.de} \\
}
\usepackage{xspace}

 \newcommand{\algname}{\textsc{CyBeR-0}\xspace}
 
\newcommand{\Trmean}{\operatorname{TrMn}}

\newcommand{\gij}{\ensuremath{\mathbf{g}_{i,j}}}
\newcommand{\gi}{\ensuremath{\mathbf{g}_{i}}}
\newcommand{\ginorm}{\ensuremath{g_i}}
\newcommand{\ggamma}{\ensuremath{\mathbf{g}}}
\newcommand{\ggammanorm}{\ensuremath{g}}
\newcommand{\gpop}{\ensuremath{\bar{\mathbf{g}}}}
\newcommand{\gpopnorm}{\bar{\ensuremath{g}}}
\newcommand{\ghat}{\ensuremath{\hat{\mathbf{g}}^{t}}}
\newcommand{\dataspace}{\ensuremath{\Gamma}}
\newcommand{\muextraarg}{\ensuremath{}}
\newcommand{\zeroextraarg}{\ensuremath{, 0}}
\newcommand{\eg}{\emph{, e.g.,}\xspace}
\newcommand{\ie}{\emph{, i.e.,}\xspace}

\newcommand{\define}{\triangleq}

\DeclareMathOperator*{\argmin}{argmin}

\newcommand{\bfg}{{\boldsymbol g}}

\newcommand{\bfw}{{\boldsymbol w}}
\newcommand{\bfx}{{\boldsymbol x}}

\newcommand{\bfz}{{\boldsymbol z}}

\newcommand{\bgamma}{\boldsymbol{\gamma}}

\newcommand{\cD}{\mathcal{D}}

\newcommand{\cW}{\mathcal{W}}

\newcommand{\twonorm}[1]{\left\|#1\right\|_2}

\newcommand{\inabs}[1]{\ensuremath{\left|#1\right|}}
\newcommand{\inset}[1]{\ensuremath{\left\{#1\right\}}}

\newcommand{\st}{\ensuremath{\,:\,}}

\newcommand{\rvD}{\mathsf{D}}

\begin{document}

\maketitle

\begin{abstract}

  We introduce \algname, the first zero-order optimization algorithm for memory-and-communication efficient Federated Learning, resilient to Byzantine faults. We show through extensive numerical experiments on the MNIST dataset
  and fine-tuning RoBERTa-Large that \algname outperforms state-of-the-art algorithms in terms of communication and memory efficiency while reaching similar accuracy. We provide theoretical guarantees on its convergence for convex loss functions.
  \end{abstract}

  \section{Introduction}
Federated Learning \cite{mcmahan2017communication}  is a distributed machine learning approach where a shared model is trained across multiple devices or servers, each holding their own local data, and a central server, called the federator. Rather than transferring the data itself, the devices (called clients) collaboratively update and improve a central model by exchanging model parameters.

Federated Learning (FL) faces several key challenges, such as heterogeneity across clients' data \cite{ZhaoFedNonIID2018,zhu2021federated}, maintaining the privacy of the clients' data \cite{schlegel2023codedpaddedfl,Egger2023Private,jahani2023swiftagg+,zhu2019applying,truex2019hybrid,bagdasaryan2019differential,wei2020federated,tang2024private}, security against Byzantine clients \cite{shen2016auror, blanchardMachineLearningAdversaries2017,li2020learning,yinByzantineRobustDistributedLearning2021,xhemrishi2023balancing} and communication efficiency \cite{wen2017terngrad,karimireddy2019error,vogels2019powersgd,m2021efficient,makkuvaLASERLinearCompression2023,tang2023z,qin2023federated}. We focus on jointly maintaining security and communication efficiency, two challenges that have been studied separately in the literature.

Security against Byzantine clients\ie clients deliberately corrupting their computation to disrupt the process, is paramount. It is shown in~\cite{blanchardMachineLearningAdversaries2017} that, if undetected, only one Byzantine client inserting errors is enough to prevent the learning algorithm from converging.

In FL settings, clients run backpropagation on their local data and communicate a $d$-dimensional vector (usually $d$ is in the order of $10^6$--$10^9$) to the federator, making communication and computation two crucial bottlenecks, especially in settings where the clients are edge devices with limited computation and communication capabilities and energy constraints. %

We introduce \algname, a novel \underline{c}ommunication-\underline{e}fficient and \underline{By}zantine-\underline{r}esilient algorithm within the context of \underline{zero}-order optimization. Zero-order optimization, \emph{i.e.}, optimization techniques that require no gradient computation, has seen significant traction in recent years\eg \cite{salimansEvolutionStrategiesScalable2017,ilyas2018black,malladiFineTuningLanguageModels2023,fangCommunicationEfficientStochasticZerothOrder2022,fang2022communication}. Zero-order optimization alleviates both computation and communication costs by allowing clients to send approximations of their gradients using random linear projections as \emph{(i)} it does not require backpropagation; and \emph{(ii)} it allows significant compression of the communicated vector. 

While \cite{fang2022communication} achieves communication efficiency through reduced frequencies of model exchanges, we achieve significant communication cost savings by a novel bi-directional shared seed concept, inspired by \cite{salimansEvolutionStrategiesScalable2017}. As a result, clients and the federator only send \emph{$k$ real numbers}, for a parameter $k>0$, instead of high-dimensional vectors. This is coupled with an effective robust aggregation technique specifically tailored to the zero-order optimization context. 

We validate the effectiveness of \algname through extensive experiments. For MNIST, we observe the same accuracy as in~\cite{yinByzantineRobustDistributedLearning2021} with a hundred-fold saving in communication. In fine-tuning Large Language Models (LLMs) \cite{malladiFineTuningLanguageModels2023}, \algname achieves high accuracy in the presence of Byzantine clients, together with a million-fold saving in terms of communication compared to uncompressed transmission. The superiority of \algname in fine-tuning LLMs is motivated by the findings of \cite{salimansEvolutionStrategiesScalable2017} and \cite{malladiFineTuningLanguageModels2023} in which it is shown that certain problem domains exhibit an inherently low \textit{intrinsic} dimensionality, making zero-order optimization (hence also \algname) a suitable optimization technique.
The fast convergence of \algname aligns with those findings. Furthermore, it was infeasible to compare to state-of-the-art robust FL schemes that require backpropagation as the simulations overloaded the memory of our GPUs of type Nvidia GeForce RTX 4090; highlighting the memory efficiency of \algname. In addition, we theoretically prove that \algname converges under the assumption of strong convexity. 

To our knowledge, this is the first work that simultaneously addresses security and communication/memory efficiency in FL through zero-order methods; offering a unique advantage, %
particularly in environments where first-order gradient information is not available or not computationally feasible due to memory-constrained client devices.

\section{Problem Setting}\label{sec:prob-setting}
We start with notation and conventions used in the paper. 

\textbf{Notation.} Vectors are represented by boldface letters\eg \(\bfz\) %
and sets are denoted by calligraphic letters \eg \(\mathcal{W}\). The \(L_2\) norm of a vector \(\bfw\) is denoted by \(\|\bfw\|_2\). The inner product of
vectors \(\bfz\) and \(\bfw\) is represented interchangeably by
\(\left\langle\bfz, \bfw\right\rangle\) and \(\bfz^\intercal\bfw\), chosen to enhance the clarity of exposition. The $i$-the coordinate of a vector $\bfw$ is denoted by $\bfw^{(i)}$. The natural logarithm is represented by \(\log\).

We define \([n]\define \inset{1,\cdots,n}\) and $\mathbb{S}^d$ as the set of unit-norm vectors in $\mathbb{R}^d$\ie $\mathbb{S}^d = \left\{\bfx \in \mathbb{R}^d \st \twonorm{\bfx} = 1\right\}$. Given a convex set $\mathcal{W} \subseteq \mathbb{R}^d$, the Euclidean projection of a vector $\bfw\in\mathbb{R}^d$ on $\mathcal{W}$ is $\Pi_\cW(\bfw) \define \argmin_{\bfw^\prime \in  \mathcal{W}} \Vert \bfw^\prime- \bfw \Vert_2^2$. %
We denote the operation of uniformly sampling a vector $\bfz$ from $\mathbb{S}^d$ as $\bfz \sim \mathbb{S}^d$. Similarly, sampling independently and uniformly a set of vectors $Z = \left\{\bfz_1, \bfz_2, \dots, \bfz_n\right\}$ is denoted as either $Z \sim \mathbb{S}^d$ or $\bfz_1,\bfz_2, \dots, \bfz_n \sim \mathbb{S}^d$.

\textbf{Federated Learning.} The model consists of a network comprising of a federator and $m$ clients, denoted by the indices
$\left\{1, 2, \ldots, m\right\}$. Each client $i$ has a set of data samples  $ \cD_i \define \left\{\bgamma^{i,1},\ldots,\bgamma^{i,\inabs{\cD_i}}\right\} \subseteq \Gamma$, where $\dataspace$ is the data space.
Let $\cD \define \bigcup_{i\in[m]} \cD_i$ be the global dataset. The loss function, denoted by $f(\bfw;{\bgamma})$, is defined with respect to the
parameter vector $\bfw \in \mathcal{W} \subseteq \mathbb{R}^d$ and a data vector ${\bgamma}\in \dataspace$.

The federator exchanges messages with clients over multiple rounds with the goal of minimizing the statistical loss $\hat{F}(\bfw) = \frac{1}{\inabs{\cD}}\sum_{\bgamma \in \cD} f(\bfw,\bgamma)$. At training step $t \geq 0$, the federator sends a model vector $\bfw^t$ to the clients.  Each client $i$ then computes an update message $\bfg_i^t$ based on the federator's model vector $\bfw^t$ and its local dataset $\cD_i$ and sends the result back to the federator. Finally, the federator updates the model to $\bfw^{t+1}$ as a function of $\bfw^t$ and $\{\bfg^t_i:i\in[m]\}$. At the end of training step $T$, the federator outputs its learned model $\bfw^T$.

\textbf{Adversarial Model.} Byzantine clients behave as in \cref{def:attack}. A fraction $\alpha m$ (unknown to the federator), $0\leq \alpha < 1/2$, of the clients are Byzantine. Byzantine clients' indices are denoted by $\mathcal{B}$. 

\begin{definition}[Attack Model]\label{def:attack}
A Byzantine client $b \in \mathcal{B}$ has complete knowledge of the vectors transmitted by all other clients to the federator at each training step. Given this knowledge, it may send arbitrary vectors to the federator, denoted as $*$, aiming to disrupt the optimization procedure.
\end{definition}

\section{Robust Zero-Order Federated Learning}
We introduce \algname, a Byzantine-resilient federated zero-order optimization. \algname builds on the
principles of MeZO \cite{malladiFineTuningLanguageModels2023}, a memory-efficient centralized algorithm. In addition to providing Byzantine resilience, \algname significantly improves the communication efficiency compared to FedZO \cite{fang2022communication}, the FL counterpart of MeZO. \algname is given in Algorithm~\ref{alg:theoretical} and is explained next. 

While the components of \algname are separately well-studied, their successful combination yields subtle but significant complexities. Extending centralized zero-order algorithms to Byzantine-resilient FL requires ensuring that the choice of the vectors $\bfz_1,\cdots,\bfz_k$ does not compromise resiliency and still allows convergence. In addition, proving that the robustness of a trimmed mean operation, meant to operate on the coordinates of gradient vectors, extends to the zero-order estimator requires novel theoretical tools. Further challenges include practical validation of the proposed algorithm and investigating zero-order fine-tuning in federated large language models. We provide next a formal explanation of the main concepts and the algorithm.

\begin{algorithm}[t]
    \caption{\algname}
    \label{alg:theoretical}
    \begin{algorithmic}
        \STATE {\bfseries Input: } $\bfw_0 \in \mathcal{W}$ is the initial model parameter vector, $\beta$ is the trimmed mean factor, $\eta$ is the learning rate, $\mu$ is the perturbation step, $k$ is the number of samples per estimate and $T$ is the total number of learning steps.
        \FOR{$t=0$ \textbf{to} $T$}
        \STATE $\blacktriangleright$ \textbf{Federator}
        \STATE $\;\;\;$ Samples $Z^t = \left\{\bfz_1^t, \bfz_2^t, \dots, \bfz_k^t\right\} \sim \mathbb{S}^d$
        \STATE $\;\;\;$ Distributes $\bfw^t = \bfw^t$ to each client
        \FOR{$i=0$ \textbf{to} $m$ \textbf{in parallel}}
        \STATE $\triangleright$ \textbf{client} $i$
        \STATE $\;\;\;$ \textbf{for each} $\bfz_r^t\in Z^t$ \textbf{do}
        \STATE $\;\;\;\;\;\;$ Compute and send $\ginorm(\bfw^t, \bfz_r^t\muextraarg)$ to federator, cf. \cref{def:zero_order}
        \STATE $\;\;\;$ \textbf{end for}
        \ENDFOR
        \STATE $\blacktriangleright$ \textbf{Federator}
        \STATE $\;\;\;$ $\ghat \leftarrow \frac{1}{k}\sum_{r=1}^k \hat{g}_{\beta}(\bfw^t, \bfz_r^t\muextraarg)\bfz_r^t$
        \STATE $\;\;\;$ $\bfw^{t+1} \leftarrow \Pi_\mathcal{W}(\bfw^t - \eta_t \ghat)$
        \ENDFOR
    \end{algorithmic}
\end{algorithm}

\begin{definition}[Zero-Order estimate]\label{def:zero_order}
    For a given $\mu\geq 0$, a vector $\bfz \in \mathbb{S}^d$ and a loss function $f(\mathbf{w , {\bgamma}})$, the \emph{single-sample zero-order estimate} of the gradient $\nabla f(\mathbf{w , {\bgamma}})$ is defined  as
\begin{align*}
    \mathbf{g}(\bfw, \bfz, \bgamma, \mu) =
    \begin{cases}
        d\frac{f(\bfw + \mu\bfz ; {\bgamma}) - f(\bfw - \mu\bfz , {\bgamma})}{2\mu}\bfz, & \text{for } \mu > 0, \\
        d\left\langle \nabla f(\mathbf{w , {\bgamma}}), \bfz \right\rangle \bfz,                                            & \text{for } \mu = 0.
    \end{cases}\end{align*}
\end{definition}

As we consider a fixed $\mu$ throughout the paper, we write $\mathbf{g}(\bfw, \bfz,\bgamma)$ instead of $\mathbf{g}(\bfw, \bfz, \bgamma, \mu)$.

To simplify notation, for each client $i \in [m]$ and each $j \in [\inabs{\cD_i}]$, we define the client partial estimate as $\gij(\bfw, \bfz) \define \mathbf{g}(\bfw, \bfz, \bgamma^{i,j})$ and the \emph{client estimate} as 
    $\gi(\bfw, \bfz) \define \frac{1}{\inabs{\cD_i}}\sum_{j=1}^{\inabs{\cD_i}} \gij(\bfw, \bfz).$ 
We define the norms $\ginorm(\bfw, \bfz) \define \twonorm{\gi(\bfw, \bfz)}$ and $\ggammanorm(\bfw, \bfz, \bgamma, \mu) \define \mathbf{g}(\bfw, \bfz, \bgamma, \mu)$.

\begin{definition}[Trimmed Mean (adopted from \cite{yinByzantineRobustDistributedLearning2021})]
    Given $0\leq \beta <1/2$ and a multiset $X = \left\{x_1, x_2, \cdots, x_m\right\}$, the trimmed mean operation is defined as
    $ %
        \Trmean_\beta\left\{X\right\} \define \frac{1}{m - 2\lfloor \beta m \rfloor}\sum_{x\in X_\beta} x,
    $ %
    where $X_\beta$ is obtained by removing the largest and smallest $\lfloor \beta m \rfloor$ elements from $X$.
\end{definition}

\textbf{Computation Procedure.} At training step $t$, the federator shares the model $\bfw^t$ and $k$ vectors $\bfz_1^{t}, \dots, \bfz_k^{t}\sim \mathbb{S}^d$ with the clients. Client $i$ computes $\gi(\bfw^t,\bfz_r^t)$ and sends $\ginorm(\bfw^t,\bfz_r^t)$, $r\in [k]$, to the federator. Byzantine clients send arbitrary vectors denoted by *. The robust aggregation proceeds as follows. For each $r\in[k]$, the federator first computes $\hat{g}_{\beta}(\bfw^t, \bfz_r^t)$ as
$$\hat{g}_{\beta}(\bfw^t, \bfz_r^t) \define \Trmean_{\beta}\left\{\left\{\ginorm(\bfw^t,\bfz_r^t)%
:i \in [m] \setminus \mathcal{B} \right\} \cup \left\{*: i\in \mathcal{B}\right\}\right\}.$$ 
Next, the federator computes the gradient estimate $\ghat = \frac{1}{k}\sum_{r=1}^k \hat{g}_{\beta}(\bfw^t, \bfz_r^t) \bfz_r^t$ and updates the model as $\bfw^{t+1} = \Pi_\mathcal{W}(\bfw^t - \eta_t \ghat)$ for a given learning rate $\eta_t$. To reduce the communication cost, the federator can only broadcast $\hat{g}_{\beta}(\bfw^t, \bfz_r^t), \forall r \in [k]$ and each client reconstruct $\bfw^{t+1}$ individually.

\subsection{Properties of \algname}
Our proposed algorithm exhibits Byzantine resilience and communication and memory efficiency. Those are notable properties contributing to its effectiveness in federated learning.

\textbf{Byzantine Resilience.} A key component of our robust aggregation is that it aligns with the compression mechanism used by the clients. The transmitted values can be perceived as a linear $k$-dimensional compressed representation of the gradient. This format is inherently amenable to scalar robustness procedures, such as the trimmed mean.

\textbf{Communication Efficiency.} \algname compresses $d$-dimensional vector into $k$ real values, offering low communication costs. At first glance, \algname appears to rely on transmitting the $d$-dimensional vectors $\bfz_1^t,\dots,\bfz_k^t$ at each training step. However, this overhead is efficiently mitigated through a simple yet effective strategy: utilizing a shared common seed among clients and the federator. The federator disseminates this seed, which is then used by the clients to sample the perturbation directions $\bfz_1^t,\dots,\bfz_k^t$ in a coordinated manner. Local updates (such as those in \cite{fang2022communication}) can further reduce the communication cost. For clarity of exposition, we provide \algname with local updates in \cref{alg:practical-le} in the appendix, together with corresponding numerical experiments. %

\textbf{Memory Efficiency.} By using Zero-order approximation, which inherently does not necessitate backpropagation, \algname saves significantly on memory (by up to a factor of $12$) compared to traditional training methods relying on backpropagation, cf. \cite{malladiFineTuningLanguageModels2023}. %
Furthermore, \algname adopts in-place perturbations on model parameters, a technique also used in MeZO \cite{malladiFineTuningLanguageModels2023}, to further reduce memory usage. 

Under a different setting, $\mu = 0$, \algname's gradient estimates are computed by projecting the true gradient along different directions. This removes the memory efficiency property in exchange for better computational efficiency, since instead of calculating $2k$ function evaluations, only a single gradient and $k$ projection calculations are needed.

To the best of our knowledge, \algname is the first application of Byzantine resilience in zero-order compressed information scenarios. In contrast, for any case other than $k=1$, coordinate-wise methods \cite{wright2015coordinate} do not allow the same compression mechanism to occur in the federator-to-client communication.

Despite its advantages, our algorithm provides a weak privacy guarantee of the clients' data ensured by the fact the clients only transmit a harshly compressed version of the gradient estimate. Ensuring strong privacy guarantees are known to have a tension with ensuring robustness \cite{dong2021oblivious,lu2023robust,xia2024byzantineresilient}. Investigating privacy, robustness and compression is a very interesting future research direction. 
While fairness can be favored by robust aggregation and enhanced through regularizers to the loss function, pruning outliers may pose additional challenges. Thoroughly analyzing the fairness guarantee is out of the scope of this work.

\section{Experiments}
\label{sec:experiments}
We present a series of experiments showing the performance of \algname across various scenarios. 
First, we %
employ a Logistic Regression model on MNIST \cite{lecun1998gradient} to investigate
the parameters' influence on the convergence of \algname. This setting serves as a
foundational test with insights into the baseline performance and parameter sensitivities
of \algname in a controlled environment. For a more comprehensive understanding of \algname's performance reaching towards advanced applications in Natural Language Processing (NLP), we extend our examination to federated fine-tuning of large language models (LLMs). %

While Algorithm \ref{alg:theoretical} completely specifies the behavior of \algname, Appendix \ref{app:algorthm} presents an extended algorithm utilized for our experiments, highlighting the memory and communication optimizations in our implementation. The basic functionality of the algorithm remains unchanged.

\subsection{Experimental Setup}
The key conditions for our experiments are client data distribution and simulating Byzantine behaviors, explained next.  General simulation parameters and hyperparameters are shown in Appendix \ref{app:params}. Tables and figures displaying standard deviation measures (denoted by the $\pm$ sign) represent the average outcomes of three independent simulation runs, initialized by different random seeds.

\textbf{Data Distribution.} We investigate two distinct data distribution scenarios: independent and identically distributed (IID) data and non-IID data. In the IID scenario, data labels are uniformly distributed across all participating clients, ensuring an equal representation of each label in the local datasets. Conversely, in the non-IID scenario, we assign a unique label set to each client, thereby creating
a skewed label distribution and introducing additional complexity in
the learning process.

\textbf{Byzantine Behavior.} We consider a worst-case scenario wherein Byzantine clients
are fully aware of the communication protocol and the transmissions of other clients.
Byzantine clients can collude and act in a deliberately adversarial manner.
Our approach draws inspiration from the attack model described in \cite{fang2020local},
\emph{i.e.}, focusing on maximizing the local gradient deviation at each training step. This is
achieved by strategic manipulation of the information sent to the federator. %
In this attack, the Byzantine clients compute the true gradient estimate obtained from the honest clients. If that estimate is positive, they all send a value equal to the $\lfloor \beta m \rfloor$-th smallest honest gradient value. Otherwise, they all send a value equal to the $\lfloor \beta m \rfloor$-th largest honest gradient value. This attack is called Full Knowledge and described in Algorithm~\ref{alg:attack}.

For the MNIST experiments, we compare this choice to other Byzantine behaviors to display its effectiveness. In particular, we compare it to three model poisoning strategies: Always Small, Always Large, and Random Choice, in which Byzantine clients all send either the $\lfloor \beta m\rfloor$-th smallest, largest or randomly pick one of them, respectively, for each perturbation direction. And a data positioning strategy, Label Flipping, in which the Byzantine devices switch each MNIST label from $\ell$ to $9-\ell$.

\begin{figure*}[t]
    \centering
    \subfigure[Accuracy vs. Communication.]{\begin{minipage}{0.32\textwidth}\resizebox{\textwidth}{!}{\begin{tikzpicture}
                \begin{axis}[
                        xlabel={Communication Cost (in scalars)},
                        ylabel={Test Accuracy},
                        legend pos=south east,
                        grid=both,
                        xmin=1, xmax=3140000,
                        ymin=-19,
                        xmode=log,log basis x=10,
                         xmajorgrids=true, 
                        xminorgrids=false, 
                        scale only axis,
                        height=0.67\textwidth,
                        width=\textwidth,
                        legend style={nodes={scale=0.7, transform shape}}, 
                        xtick pos=left,
                    ]

\addplot [blue, thick] table[col sep=comma, x expr=7850*\thisrow{Epoch} + 1, y=Loss] {data/k_comp/acc_iidinf_BYZ0_LR0.01True_gradient__b=0__mode=MIN_ACCEPTED.csv};
                    \addlegendentry{FedAvg}

\addplot [purple, thick] table[col sep=comma, x expr=64*\thisrow{Epoch} + 1, y=Loss] {data/k_comp/acc_iidinf_BYZ0_LR0.01k=64__b=0__mode=MIN_ACCEPTED.csv};
                    \addlegendentry{$k=64$}

\addplot [green, thick] table[col sep=comma, x expr=16*\thisrow{Epoch} +1, y=Loss] {data/k_comp/acc_iidinf_BYZ0_LR0.01k=16__b=0__mode=MIN_ACCEPTED.csv};
                    \addlegendentry{$k=16$}

\addplot [red, thick] table[col sep=comma, x expr=4*\thisrow{Epoch}+1, y=Loss] {data/k_comp/acc_iidinf_BYZ0_LR0.01k=4__b=0__mode=MIN_ACCEPTED.csv};
                    \addlegendentry{$k=4$}

\addplot [orange, thick] table[col sep=comma, x expr=\thisrow{Epoch} + 1, y=Loss] {data/k_comp/acc_iidinf_BYZ0_LR0.01k=1__b=0__mode=MIN_ACCEPTED.csv};
                    \addlegendentry{$k=1$}

                    \addplot [name path=upper,draw=none] table[col sep=comma, x expr=7850*\thisrow{Epoch}+1, y expr=\thisrow{Loss}+\thisrow{StdDev}] {data/k_comp/acc_iidinf_BYZ0_LR0.01True_gradient__b=0__mode=MIN_ACCEPTED.csv};
                    \addplot [name path=lower,draw=none] table[col sep=comma, x expr=7850*\thisrow{Epoch}+1, y expr=\thisrow{Loss}-\thisrow{StdDev}] {data/k_comp/acc_iidinf_BYZ0_LR0.01True_gradient__b=0__mode=MIN_ACCEPTED.csv};
                    \addplot [fill=blue!10] fill between[of=upper and lower];

                    \addplot [name path=upper,draw=none] table[col sep=comma, x expr=1*\thisrow{Epoch}+1, y expr=\thisrow{Loss}+\thisrow{StdDev}] {data/k_comp/acc_iidinf_BYZ0_LR0.01k=1__b=0__mode=MIN_ACCEPTED.csv};
                    \addplot [name path=lower,draw=none] table[col sep=comma, x expr=1*\thisrow{Epoch}+1, y expr=\thisrow{Loss}-\thisrow{StdDev}] {data/k_comp/acc_iidinf_BYZ0_LR0.01k=1__b=0__mode=MIN_ACCEPTED.csv};
                    \addplot [fill=orange!10] fill between[of=upper and lower];

                    \addplot [name path=upper,draw=none] table[col sep=comma, x expr=4*\thisrow{Epoch}+1, y expr=\thisrow{Loss}+\thisrow{StdDev}] {data/k_comp/acc_iidinf_BYZ0_LR0.01k=4__b=0__mode=MIN_ACCEPTED.csv};
                    \addplot [name path=lower,draw=none] table[col sep=comma, x expr=4*\thisrow{Epoch}+1, y expr=\thisrow{Loss}-\thisrow{StdDev}] {data/k_comp/acc_iidinf_BYZ0_LR0.01k=4__b=0__mode=MIN_ACCEPTED.csv};
                    \addplot [fill=red!10] fill between[of=upper and lower];

                    \addplot [name path=upper,draw=none] table[col sep=comma, x expr=16*\thisrow{Epoch}+1, y expr=\thisrow{Loss}+\thisrow{StdDev}] {data/k_comp/acc_iidinf_BYZ0_LR0.01k=16__b=0__mode=MIN_ACCEPTED.csv};
                    \addplot [name path=lower,draw=none] table[col sep=comma, x expr=16*\thisrow{Epoch}+1, y expr=\thisrow{Loss}-\thisrow{StdDev}] {data/k_comp/acc_iidinf_BYZ0_LR0.01k=16__b=0__mode=MIN_ACCEPTED.csv};
                    \addplot [fill=green!10] fill between[of=upper and lower];

                    \addplot [name path=upper,draw=none] table[col sep=comma, x expr=64*\thisrow{Epoch}+1, y expr=\thisrow{Loss}+\thisrow{StdDev}] {data/k_comp/acc_iidinf_BYZ0_LR0.01k=64__b=0__mode=MIN_ACCEPTED.csv};
                    \addplot [name path=lower,draw=none] table[col sep=comma,  x expr=64*\thisrow{Epoch}+1, y expr=\thisrow{Loss}-\thisrow{StdDev}] {data/k_comp/acc_iidinf_BYZ0_LR0.01k=64__b=0__mode=MIN_ACCEPTED.csv};
                    \addplot [fill=purple!10] fill between[of=upper and lower];
                \end{axis}
            \end{tikzpicture}}\end{minipage}\label{fig:k-comparison-mnist-comm}} 
    \subfigure[Accuracy vs. Steps.]{\begin{minipage}{0.32\textwidth}\resizebox{\textwidth}{!}{\begin{tikzpicture}
                \begin{axis}[
                        xlabel={Steps},
                        ylabel={Test Accuracy},
                        legend pos=south east,
                        grid=both,
                        xmin=1, xmax=400,
                        legend style={nodes={scale=0.7, transform shape}}, 
                        scale only axis,
                        height=0.67\textwidth,
                        width=\textwidth,
                    ]

\addplot [blue, thick] table[col sep=comma, x expr=\thisrow{Epoch} + 1, y=Loss] {data/k_comp/acc_iidinf_BYZ0_LR0.01True_gradient__b=0__mode=MIN_ACCEPTED.csv};
                    \addlegendentry{FedAvg}

\addplot [purple, thick] table[col sep=comma, x expr=\thisrow{Epoch} + 1, y=Loss] {data/k_comp/acc_iidinf_BYZ0_LR0.01k=64__b=0__mode=MIN_ACCEPTED.csv};
                    \addlegendentry{$k=64$}

\addplot [green, thick] table[col sep=comma, x expr=\thisrow{Epoch} +1, y=Loss] {data/k_comp/acc_iidinf_BYZ0_LR0.01k=16__b=0__mode=MIN_ACCEPTED.csv};
                    \addlegendentry{$k=16$}

\addplot [red, thick] table[col sep=comma, x expr=\thisrow{Epoch}+1, y=Loss] {data/k_comp/acc_iidinf_BYZ0_LR0.01k=4__b=0__mode=MIN_ACCEPTED.csv};
                    \addlegendentry{$k=4$}

\addplot [orange, thick] table[col sep=comma, x expr=\thisrow{Epoch} + 1, y=Loss] {data/k_comp/acc_iidinf_BYZ0_LR0.01k=1__b=0__mode=MIN_ACCEPTED.csv};
                    \addlegendentry{$k=1$}

                    \addplot [name path=upper,draw=none] table[col sep=comma, x expr=\thisrow{Epoch}+1, y expr=\thisrow{Loss}+\thisrow{StdDev}] {data/k_comp/acc_iidinf_BYZ0_LR0.01True_gradient__b=0__mode=MIN_ACCEPTED.csv};
                    \addplot [name path=lower,draw=none] table[col sep=comma, x expr=\thisrow{Epoch}+1, y expr=\thisrow{Loss}-\thisrow{StdDev}] {data/k_comp/acc_iidinf_BYZ0_LR0.01True_gradient__b=0__mode=MIN_ACCEPTED.csv};
                    \addplot [fill=blue!10] fill between[of=upper and lower];

                    \addplot [name path=upper,draw=none] table[col sep=comma, x expr=\thisrow{Epoch}+1, y expr=\thisrow{Loss}+\thisrow{StdDev}] {data/k_comp/acc_iidinf_BYZ0_LR0.01k=1__b=0__mode=MIN_ACCEPTED.csv};
                    \addplot [name path=lower,draw=none] table[col sep=comma, x expr=\thisrow{Epoch}+1, y expr=\thisrow{Loss}-\thisrow{StdDev}] {data/k_comp/acc_iidinf_BYZ0_LR0.01k=1__b=0__mode=MIN_ACCEPTED.csv};
                    \addplot [fill=orange!10] fill between[of=upper and lower];

                    \addplot [name path=upper,draw=none] table[col sep=comma, x expr=\thisrow{Epoch}+1, y expr=\thisrow{Loss}+\thisrow{StdDev}] {data/k_comp/acc_iidinf_BYZ0_LR0.01k=4__b=0__mode=MIN_ACCEPTED.csv};
                    \addplot [name path=lower,draw=none] table[col sep=comma, x expr=\thisrow{Epoch}+1, y expr=\thisrow{Loss}-\thisrow{StdDev}] {data/k_comp/acc_iidinf_BYZ0_LR0.01k=4__b=0__mode=MIN_ACCEPTED.csv};
                    \addplot [fill=red!10] fill between[of=upper and lower];

                    \addplot [name path=upper,draw=none] table[col sep=comma, x expr=\thisrow{Epoch}+1, y expr=\thisrow{Loss}+\thisrow{StdDev}] {data/k_comp/acc_iidinf_BYZ0_LR0.01k=16__b=0__mode=MIN_ACCEPTED.csv};
                    \addplot [name path=lower,draw=none] table[col sep=comma, x expr=\thisrow{Epoch}+1, y expr=\thisrow{Loss}-\thisrow{StdDev}] {data/k_comp/acc_iidinf_BYZ0_LR0.01k=16__b=0__mode=MIN_ACCEPTED.csv};
                    \addplot [fill=green!10] fill between[of=upper and lower];

                    \addplot [name path=upper,draw=none] table[col sep=comma, x expr=\thisrow{Epoch}+1, y expr=\thisrow{Loss}+\thisrow{StdDev}] {data/k_comp/acc_iidinf_BYZ0_LR0.01k=64__b=0__mode=MIN_ACCEPTED.csv};
                    \addplot [name path=lower,draw=none] table[col sep=comma,  x expr=\thisrow{Epoch}+1, y expr=\thisrow{Loss}-\thisrow{StdDev}] {data/k_comp/acc_iidinf_BYZ0_LR0.01k=64__b=0__mode=MIN_ACCEPTED.csv};
                    \addplot [fill=purple!10] fill between[of=upper and lower];
                \end{axis}
\end{tikzpicture}}\end{minipage}\label{fig:k-comparison-mnist-steps}} 
    \subfigure[Accuracy vs. Attack model.]{\begin{minipage}{0.32\textwidth}\resizebox{\textwidth}{!}{\begin{tikzpicture}
                \begin{axis}[
                    xlabel={Steps},
                    ylabel={Validation Accuracy},
                    legend style={nodes={scale=0.7, transform shape}}, 
                    legend pos=south east,
                    grid=both,
                    xmin=0, xmax=400,
                    ymin=60, ymax=95,
                    scale only axis,
                    height=0.67\textwidth,
                    width=\textwidth,
                    ]

                    \addplot [blue, thick] table[col sep=comma, x=Epoch, y=Acc] {data/byz_comparison2/EPOCHS400-TOTAL_WORKERS100-BATCHES_PER_WORKER1-BYZANTINE_WORKERS25-MU0_001-MODEOURS-K64-BYZ_MODEMAX_ACCEPTED-DATA_DISTRIBUTIONnon-ii.csv};
                    \addlegendentry{Always Large}

                    \addplot [brown, thick] table[col sep=comma, x=Epoch, y=Acc] {data/byz_comparison2/EPOCHS400-TOTAL_WORKERS100-BATCHES_PER_WORKER1-BYZANTINE_WORKERS25-MU0_001-MODEOURS-K64-BYZ_MODEMIN_ACCEPTED-DATA_DISTRIBUTIONnon-ii.csv};
                    \addlegendentry{Always Small}

                    \addplot [orange, thick] table[col sep=comma, x=Epoch, y=Acc] {data/byz_comparison2/EPOCHS400-TOTAL_WORKERS100-BATCHES_PER_WORKER1-BYZANTINE_WORKERS25-MU0_001-MODEOURS-K64-BYZ_MODERAND_ACCEPTED-DATA_DISTRIBUTIONnon-ii.csv};
                    \addlegendentry{Random Choice}

                    \addplot [green, thick] table[col sep=comma, x=Epoch, y=Acc] {data/byz_comparison2/EPOCHS400-TOTAL_WORKERS100-BATCHES_PER_WORKER1-BYZANTINE_WORKERS25-MU0_001-MODEOURS-K64-BYZ_MODELABEL_FLIPPING-DATA_DISTRIBUTIONnon-ii.csv};
                    \addlegendentry{Label Flipping}

                    \addplot [red, thick] table[col sep=comma, x=Epoch, y=Acc] {data/byz_comparison2/EPOCHS400-TOTAL_WORKERS100-BATCHES_PER_WORKER1-BYZANTINE_WORKERS25-MU0_001-MODEOURS-K64-BYZ_MODEFK-DATA_DISTRIBUTIONnon-ii.csv};
                    \addlegendentry{Full-Knowledge}

                    \addplot [name path=upper,draw=none] table[col sep=comma, x expr=\thisrow{Epoch}, y expr=\thisrow{Acc}+\thisrow{StdDev}] {data/byz_comparison2/EPOCHS400-TOTAL_WORKERS100-BATCHES_PER_WORKER1-BYZANTINE_WORKERS25-MU0_001-MODEOURS-K64-BYZ_MODEFK-DATA_DISTRIBUTIONnon-ii.csv};
                    \addplot [name path=lower,draw=none] table[col sep=comma, x expr=\thisrow{Epoch}, y expr=\thisrow{Acc}-\thisrow{StdDev}] {data/byz_comparison2/EPOCHS400-TOTAL_WORKERS100-BATCHES_PER_WORKER1-BYZANTINE_WORKERS25-MU0_001-MODEOURS-K64-BYZ_MODEFK-DATA_DISTRIBUTIONnon-ii.csv};
                    \addplot [fill=red!10] fill between[of=upper and lower];

                    \addplot [name path=upper,draw=none] table[col sep=comma, x expr=\thisrow{Epoch}, y expr=\thisrow{Acc}+\thisrow{StdDev}] {data/byz_comparison2/EPOCHS400-TOTAL_WORKERS100-BATCHES_PER_WORKER1-BYZANTINE_WORKERS25-MU0_001-MODEOURS-K64-BYZ_MODEMAX_ACCEPTED-DATA_DISTRIBUTIONnon-ii.csv};
                    \addplot [name path=lower,draw=none] table[col sep=comma, x expr=\thisrow{Epoch}, y expr=\thisrow{Acc}-\thisrow{StdDev}] {data/byz_comparison2/EPOCHS400-TOTAL_WORKERS100-BATCHES_PER_WORKER1-BYZANTINE_WORKERS25-MU0_001-MODEOURS-K64-BYZ_MODEMAX_ACCEPTED-DATA_DISTRIBUTIONnon-ii.csv};
                    \addplot [fill=blue!10] fill between[of=upper and lower];

                    \addplot [name path=upper,draw=none] table[col sep=comma, x expr=\thisrow{Epoch}, y expr=\thisrow{Acc}+\thisrow{StdDev}] {data/byz_comparison2/EPOCHS400-TOTAL_WORKERS100-BATCHES_PER_WORKER1-BYZANTINE_WORKERS25-MU0_001-MODEOURS-K64-BYZ_MODEMIN_ACCEPTED-DATA_DISTRIBUTIONnon-ii.csv};
                    \addplot [name path=lower,draw=none] table[col sep=comma, x expr=\thisrow{Epoch}, y expr=\thisrow{Acc}-\thisrow{StdDev}] {data/byz_comparison2/EPOCHS400-TOTAL_WORKERS100-BATCHES_PER_WORKER1-BYZANTINE_WORKERS25-MU0_001-MODEOURS-K64-BYZ_MODEMIN_ACCEPTED-DATA_DISTRIBUTIONnon-ii.csv};
                    \addplot [fill=brown!10] fill between[of=upper and lower];

                    \addplot [name path=upper,draw=none] table[col sep=comma, x expr=\thisrow{Epoch}, y expr=\thisrow{Acc}+\thisrow{StdDev}] {data/byz_comparison2/EPOCHS400-TOTAL_WORKERS100-BATCHES_PER_WORKER1-BYZANTINE_WORKERS25-MU0_001-MODEOURS-K64-BYZ_MODERAND_ACCEPTED-DATA_DISTRIBUTIONnon-ii.csv};
                    \addplot [name path=lower,draw=none] table[col sep=comma, x expr=\thisrow{Epoch}, y expr=\thisrow{Acc}-\thisrow{StdDev}] {data/byz_comparison2/EPOCHS400-TOTAL_WORKERS100-BATCHES_PER_WORKER1-BYZANTINE_WORKERS25-MU0_001-MODEOURS-K64-BYZ_MODERAND_ACCEPTED-DATA_DISTRIBUTIONnon-ii.csv};
                    \addplot [fill=orange!10] fill between[of=upper and lower];

                                      \addplot [name path=upper,draw=none] table[col sep=comma,  x=Epoch, y expr=\thisrow{Acc}+\thisrow{StdDev}] {data/byz_comparison2/EPOCHS400-TOTAL_WORKERS100-BATCHES_PER_WORKER1-BYZANTINE_WORKERS25-MU0_001-MODEOURS-K64-BYZ_MODELABEL_FLIPPING-DATA_DISTRIBUTIONnon-ii.csv};
                    \addplot [name path=lower,draw=none] table[col sep=comma,  x=Epoch, y expr=\thisrow{Acc}-\thisrow{StdDev}] {data/byz_comparison2/EPOCHS400-TOTAL_WORKERS100-BATCHES_PER_WORKER1-BYZANTINE_WORKERS25-MU0_001-MODEOURS-K64-BYZ_MODELABEL_FLIPPING-DATA_DISTRIBUTIONnon-ii.csv};
                    \addplot [fill=green!10] fill between[of=upper and lower];
                \end{axis}
            \end{tikzpicture}}\end{minipage}\label{fig:byz-comparison-mnist}}
    \vspace{-.3cm}
    \caption{\algname for logistic regression on MNIST under non-IID data distribution. Figures~(a) and~(b) show the convergence for varying $k$ in the absence of Byzantine clients compared to federated averaging (FedAvg). Figure~(c) shows different attacks for $k=64$ and $\alpha=\beta=0.25$.\vspace{-.2cm}}
\end{figure*}
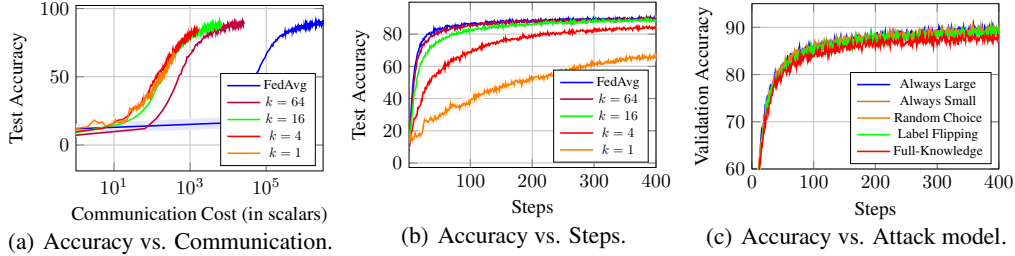

\begin{table}[b!]
    \caption{Comparison with state-of-the-art: We compare the test accuracies of \algname ($k=64$) with the trimmed mean and Krum algorithms. We use $\beta = \alpha$ for all experiments, the Full-Knowledge attack for \algname, and the model poisoning attacks described in \cite{fang2020local} for trimmed mean and Krum. }
    \vspace{-.3cm}
    \label{tab:m-comparison-state-of-the-art}
    \begin{center}
        \begin{small}
            \begin{sc}
                \begin{tabular}{lcccr}
                    \toprule
                    Algorithm &$\alpha=0.125$ & $\alpha=0.25$ &  $\alpha=0.375$\\
                    \midrule
                   Krum   & 69.2 $\pm$ 1.6      & 10.2$\pm$1.2 & 6.9$\pm$2.5   \\
                   Trimmed Mean& 86.6 $\pm$ 0.3    & 74.8 $\pm$ 1.1 &  36.0 $\pm$ 4.1  \\
                   \algname  & 87.1$\pm$ 0.8    & 80.8 $\pm$ 1.0 & 60.3 $\pm$ 2.9   \\

                    \bottomrule
                \end{tabular}
            \end{sc}
        \end{small}
    \end{center}
    \vskip -0.1in
\end{table}

\subsection{\algname with Logistic Regression on MNIST}
\label{sec:mnist-experiments}
We explore the effect of the sample size $k$ and the clients' behavior on the convergence of \algname. We then compare \algname with coordinate-wise trimmed mean \cite{yinByzantineRobustDistributedLearning2021}, which comes closest in spirit.

\textbf{Impact of Sample Size $k$.} 
The training loss trajectories for diverse settings of the sample size parameter $k$ are illustrated in Figure~\ref{fig:k-comparison-mnist-comm} over the cost of communication, and in Figure~\ref{fig:k-comparison-mnist-steps} over the number of steps. We include the conventional federated averaging (FedAvg) \cite{mcmahan2017communication} as a benchmark.

We observe a notable trend: as the value of $k$ increases, the convergence rate progressively
aligns with that of standard SGD. This aligns with theoretical expectations in Section \ref{sec:theory}, as
a larger sample size $k$ yields a sample mean that more closely approximates the true
gradient of the loss function.

\textbf{Effect of Byzantine Client Behavior.} In Figure \ref{fig:byz-comparison-mnist}, we present a comparative analysis of the training loss dynamics under different Byzantine
client behaviors. Our results show that the Full-Knowledge strategy presents the highest damage to the 
training process by causing the most substantial delay in convergence. This phenomenon underscores 
the potency of informed adversarial behaviors in disrupting the learning process.

\textbf{Comparison with State-of-the-Art.}
We compare the performance of \algname with the trimmed 
mean \cite{yinByzantineRobustDistributedLearning2021} and 
Krum \cite{blanchardMachineLearningAdversaries2017}. For this 
comparison, we apply the Full-Knowledge attack on \algname, 
while applying the model poisoning attacks from \cite{fang2020local} 
for the state-of-the-art approaches, as those are the most 
effective attacks against the respective algorithms. The results are 
shown in Table \ref{tab:m-comparison-state-of-the-art}. 

\algname provides better Byzantine resilience
for all $\alpha$ values while achieving a roughly 100-fold communication reduction. It can be assumed that the projection over the 
random directions leaves fewer degrees if freedom for the Byzantine clients to change 
the aggregated gradient, hence providing good Byzantine resilience and allowing for communication efficiency.

\subsection{Fine-Tuning Language Models with \algname}
\label{sec:experiments-ft}
Following a methodology similar to \cite{malladiFineTuningLanguageModels2023}, we utilize the RoBERTa-large model
\cite{liu2019roberta} for three distinct NLP tasks: sentiment analysis, natural language inference (NLI) and topic classification. For sentiment
analysis, we employ the SST-2 dataset \cite{socher2013recursive}. For NLI, we employ the SNLI dataset \cite{bowman2015large}. For topic classification, we use the
TREC dataset \cite{voorhees2000building}.  We adopt a prompt-based fine-tuning approach in a few-shot learning framework, as outlined by \cite{brown2020language}. Fine-tuning LLMs is well-established in the literature. The details are omitted here for brevity, interested readers are referred to \cite{brown2020language} for details. 

We use a set of $512$ data points distributed among the clients according to the specified data
distribution pattern. These experiments intend to show the applicability of \algname
in more complex and real-world scenarios, particularly in the increasingly relevant field of NLP.

\begin{table*}[t!]
    \caption{\algname in adversarial (\algname, Byzantine) and non-adversarial settings.} %
    \label{tab:main-comparison}
    \vspace{-.1cm}
    \begin{center}
        \begin{small}
            \begin{sc}
                \begin{tabular}{@{\extracolsep{8pt}}lccccr}
                    \toprule
                          & \multicolumn{1}{c}{Non-Distributed}        & \multicolumn{4}{c}{\algname}                                                                                                    \\
                    \cline{3-6}
                          & \cite{malladiFineTuningLanguageModels2023} & \multicolumn{2}{c}{Non-Byzantine} & \multicolumn{2}{c}{Byzantine}                                                                   \\
                    \cline{3-4} \cline{5-6}
                          &                                            & IID                          & Non-IID                         & IID                         & Non-IID                         \\
                    \midrule
                    SST-2 & 93.3$\pm$ 0.7                              & 93.1$\pm$ 0.3               & 93.1$\pm$ 0.2                  & 92.9$\pm$ 0.4               & 92.7 $\pm$0.4                   \\
                    TREC  & 94.3$\pm$ 1.3                              & 95.4$\pm$ 0.3    & 95.8$\pm$ 0.4 & 92.1$\pm$ 1.5 & 78.2$\pm$ 0.7 \\
                                        SNLI  & 83.0 $\pm$ 1.0                              & 84.8 $\pm$ 0.3 & 84.6 $\pm$ 0.7 & 80.0 $\pm$ 0.4 & 60.1 $\pm$ 4.9 \\
                    \bottomrule
                \end{tabular}
            \end{sc}
        \end{small}
    \end{center}
    \vskip -0.1in
\end{table*}

\begin{wrapfigure}[13]{r}{5.5cm}
    \vskip -0.3cm
    \begin{center}
    \resizebox{5cm}{!}{
        \centerline{
            \begin{tikzpicture}
                \begin{axis}[
                        xlabel={Steps},
                        ylabel={Training Loss},
                        legend pos=north east,
                        grid=both,
                        xmin=0, xmax=20000,
                        legend style={nodes={scale=0.7, transform shape}}, 
                        scale only axis,
                        height=0.21\textwidth,
                        width=0.35\textwidth,
                    ]

                    Plot for the first dataset
                    Adjust 'col sep' if your CSV uses a different separator

                    \addplot [red, thick] table[col sep=comma, x expr=10*\thisrow{Epoch}, y=Loss] {data/2_non_iid_trimmed_mean_tm_full_knowledge_0.25.csv};
                    \addlegendentry{Byzantine - Non-IID}

                    \addplot [blue, thick] table[col sep=comma, x expr=10*\thisrow{Epoch}, y=Loss] {data/2_iid_trimmed_mean_tm_full_knowledge_0.25.csv};
                    \addlegendentry{Byzantine - IID}

                    \addplot [violet, thick] table[col sep=comma, x expr=10*\thisrow{Epoch}, y=Loss] {data/0_non_iid_naive_tm_full_knowledge_0.25.csv};
                    \addlegendentry{Non-Byzantine - Non-IID}

                    \addplot [teal, thick] table[col sep=comma, x expr=10*\thisrow{Epoch}, y=Loss] {data/0_iid_naive_tm_full_knowledge_0.25.csv};
                    \addlegendentry{Non-Byzantine - IID}

                    \addplot [name path=upper,draw=none] table[col sep=comma, x expr=10*\thisrow{Epoch}, y expr=\thisrow{Loss}+\thisrow{StdDev}] {data/0_iid_naive_tm_full_knowledge_0.25.csv};
                    \addplot [name path=lower,draw=none] table[col sep=comma, x expr=10*\thisrow{Epoch}, y expr=\thisrow{Loss}-\thisrow{StdDev}] {data/0_iid_naive_tm_full_knowledge_0.25.csv};
                    \addplot [fill=teal!10] fill between[of=upper and lower];

                    \addplot [name path=upper,draw=none] table[col sep=comma, x expr=10*\thisrow{Epoch}, y expr=\thisrow{Loss}+\thisrow{StdDev}] {data/2_iid_trimmed_mean_tm_full_knowledge_0.25.csv};
                    \addplot [name path=lower,draw=none] table[col sep=comma, x expr=10*\thisrow{Epoch}, y expr=\thisrow{Loss}-\thisrow{StdDev}] {data/2_iid_trimmed_mean_tm_full_knowledge_0.25.csv};
                    \addplot [fill=blue!10] fill between[of=upper and lower];

                    \addplot [name path=upper,draw=none] table[col sep=comma, x expr=10*\thisrow{Epoch}, y expr=\thisrow{Loss}+\thisrow{StdDev}] {data/0_non_iid_naive_tm_full_knowledge_0.25.csv};
                    \addplot [name path=lower,draw=none] table[col sep=comma, x expr=10*\thisrow{Epoch}, y expr=\thisrow{Loss}-\thisrow{StdDev}] {data/0_non_iid_naive_tm_full_knowledge_0.25.csv};
                    \addplot [fill=violet!10] fill between[of=upper and lower];

                    \addplot [name path=upper,draw=none] table[col sep=comma, x expr=10*\thisrow{Epoch}, y expr=\thisrow{Loss}+\thisrow{StdDev}] {data/2_non_iid_trimmed_mean_tm_full_knowledge_0.25.csv};
                    \addplot [name path=lower,draw=none] table[col sep=comma, x expr=10*\thisrow{Epoch}, y expr=\thisrow{Loss}-\thisrow{StdDev}] {data/2_non_iid_trimmed_mean_tm_full_knowledge_0.25.csv};
                    \addplot [fill=red!10] fill between[of=upper and lower];

                \end{axis}
            \end{tikzpicture}}}
            \vspace{-.3cm}
        \caption{Effect of Byzantine clients on the convergence speed of \algname.}%
        \label{fig:sst-train-loss}
    \end{center}
\end{wrapfigure}
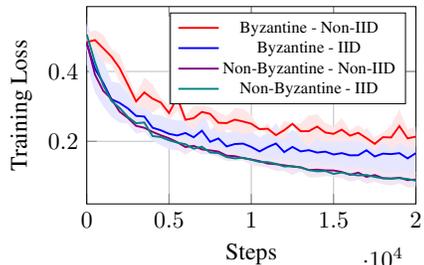
In the classical theory of zero-order optimization \cite{nesterovRandomGradientFreeMinimization2017}, fine-tuning LLMs is deemed to be of prohibitively slow convergence due to the role exercised by the model dimension $d$. Nevertheless, as evidenced by the findings of \cite{salimansEvolutionStrategiesScalable2017}  and \cite{malladiFineTuningLanguageModels2023}, certain problem domains exhibit an inherently low \textit{intrinsic} dimensionality.
The fast convergence of \algname, cf. first and second columns of \cref{tab:main-comparison}, aligns with the findings of \cite{salimansEvolutionStrategiesScalable2017} and \cite{malladiFineTuningLanguageModels2023}.

\textbf{Robustness of \algname.}
The ability of \algname to mitigate the effect of Byzantine clients, using the strong Full-Knowledge attack, can be seen in the right-most column of \cref{tab:main-comparison}. In IID settings, \algname exhibits a small drop in accuracy. However, for non-IID settings, while still converging, \algname exhibits a drop in accuracy in the presence of Byzantine clients. This behavior aligns with the literature on non-IID robust FL. The main reason is that the non-IID data distribution is reflected in the clients' message updates. Making the distinction between malicious gradients and outliers more challenging. %
In addition to affecting the final accuracy, Byzantine clients also decrease the convergence speed of the algorithm. We illustrate this effect on the SST-2 experiment in Figure \ref{fig:sst-train-loss}, which complements the data presented in Table \ref{tab:main-comparison}. %
Similar figures for the SNLI and TREC experiments are given in Appendix~\ref{app:roberta}. %

We further demonstrate communication optimizations and scalability aspects of \algname using RoBERTa-large on the SST-2 dataset. We vary the number of samples per training step and the total number of clients, cf. \cref{tab:k-comparison-sst2}. We provide in Appendix \ref{sec:local-epochs} an extra set of experiments regarding local epochs \cite{mcmahan2017communication}.

\textbf{Number of Samples.}
As observed in Table \ref{tab:k-comparison-sst2}, and consistent with
findings from Section \ref{sec:mnist-experiments}, for fixed values of $m=8$ and $\beta=0.25$, an increase in the number
of samples $k$ correlates with accelerated convergence in terms of steps. \algname presents here a tradeoff between computational workload and communication efficiency. A larger value of $k$ necessitates more forward passes per training step. However, the transmission
of these passes in batches potentially
enhances communication efficiency by reducing
the need for frequent synchronization rounds.

\begin{table}
    \caption{Varying the number of samples ($k$) and the number of clients ($m$) in SST-2 using RoBERTa-large in the presence of Byzantine clients with non-IID data distribution.} \vspace{-.3cm}
    \label{tab:k-comparison-sst2}
    \begin{center}
        \begin{small}
            \begin{sc}
                \begin{tabular}{lcc||lcc}
                    \toprule
                    $k$ & Total Steps & Test Accuracy & $m$  & Byzantine Clients & Test Accuracy \\
                    \midrule
                    1   & 20,000      & 92.7 $\pm$ $0.4$ & 8 & 2 & 92.7 $\pm$ $0.4$        \\
                    2   & 10,000      & 92.5 $\pm$ $0.3$     & 16 & 4 & 92.5 $\pm$ $0.4$  \\
                    4   & 5,000       & 92.8 $\pm$ $0.5$      & 32 & 8 & 92.3 $\pm$ $0.6$ \\
                    8   & 2,500       & 92.7 $\pm$ $0.5$      & & &\\
                    \bottomrule
                \end{tabular}
            \end{sc}
        \end{small}
    \end{center}
    \vspace{-.1cm}

\vskip -0.1in
\end{table}

\textbf{Number of Clients.}
We explore the impact of scaling the number of clients
while maintaining the same ratio of Byzantine to non-Byzantine clients. As observed in Table~\ref{tab:k-comparison-sst2}, for fixed values of $k=1$ and $\beta=0.25$, increasing the number of clients does not significantly affect the final test accuracy. This
outcome aligns with the expectation that similar data distribution among non-Byzantine clients would result
in consistent learning patterns, regardless of the network size.

\section{Theoretical Analysis}
\label{sec:theory}
To complement our extensive numerical results, we provide a theoretical convergence guarantee for \algname for convex loss functions and IID data distribution. The results under those assumptions quantify the interplay between convergence guarantee and the choice of $\mu$, $k$, and $d$ in well-behaved settings and pave the way to an extended analysis for non-convex losses and non-IID data distribution. %

Adopting an approach akin to \cite{yinByzantineRobustDistributedLearning2021}, we establish a
probabilistic bound on the distance between the robust gradient estimation yielded by our method and the ideal
expected gradient obtained in a Byzantine-free context. With such a bound established, we carry a convergence analysis for an
SGD algorithm that operates under bounded-error conditions, using the zero-order
gradient estimate.

\subsection{Preliminaries}
To establish theoretical results, we make an assumption on the data distribution and define the population loss and the zero-order population estimate.

\begin{assumption}[IID Data Distribution]
    Each client $i$ has a set of $n$ data samples $\inset{\bgamma^{i,1}, \dots, \bgamma^{i,n}}$ sampled from a common data distribution $\rvD$.
\end{assumption}

\begin{definition}[Population Loss]
     The population
loss $F(\bfw)$ is expressed as the expected value of the loss over $\cD$ , \emph{i.e.}, \vspace{-.2cm}
$%
F(\bfw) = \mathbb{E}_{{\bgamma} \sim \rvD}[f(\bfw;{\bgamma})]. 
$%
\end{definition}

Associated with the population loss, we have the following optimization problem \vspace{-.1cm}
\begin{align}
    \label{eq:main-opt}
    \bfw^\star = \argmin_{\bfw \in \mathcal{W}} F(\bfw).
\end{align}

\begin{assumption}[Local Minimum]
    \label{ass:local-minimum} The model $\bfw^*$ is a local minimum of $F$.
\end{assumption}
\begin{definition}[Zero-Order Population Estimate]
    \label{def:zo-pop-estimate}
    Let $F$ be the population loss for the optimization problem in \eqref{eq:main-opt}, then we define the zero-order population  estimate by:
\begin{align*}
    \gpop(\bfw, \bfz\muextraarg) & = \mathbb{E}_{\bgamma \sim \rvD}[\ggamma(\bfw, \bfz,\bgamma\muextraarg)]                                                       %
                                            = \begin{cases}
                                                    d\frac{F(\bfw + \mu\bfz) - F(\bfw - \mu\bfz))}{2\mu}\bfz, & \text{for } \mu > 0, \\
                                                    d\left\langle \nabla F(\bfw), \bfz \right\rangle \bfz,                           & \text{for } \mu = 0,
                                                \end{cases}
                                                \end{align*}
    and its norm by $\gpopnorm(\bfw, \bfz\muextraarg) = \|\gpop(\bfw, \bfz\muextraarg)\|_2$.
\end{definition}

\subsection{Robustness Error Bound}
We proceed by deriving the aforementioned bound.
We start by adding three assumptions on the functions $\ggammanorm$, $f$ and $F$ and the parameter space $\mathcal{W}$.
\begin{assumption}[Smoothness]
    \label{assump:smooth}
    Consider any $\mu\ge 0$, $\bgamma \in \Gamma$, $\bfw\in\mathbb{R}^d$, and $\bfz \in \mathbb{S}^d$. We assume that
    $\ggammanorm(\cdot, \bfz,\bgamma\muextraarg)$ exhibits $L_{w,\mu}$-Lipschitz continuity, that $\ggammanorm(w, \cdot,\bgamma\muextraarg)$ exhibits $L_{z,\mu}$-Lipschitz continuity
    and that $f(\cdot, \bgamma)$ is $L$-smooth. Additionally, it is assumed that $F(\cdot)$ demonstrates $L_F$-smoothness. For simplicity of notation, we denote $\hat{L}_\mu = L_{w,\mu} + L_{z,\mu}$.
\end{assumption}
\begin{assumption}[Sub-Exponentiality]
    \label{assump:subexp}
    For all $\mu \ge 0$, $\bfz \in \mathbb{S}^d$, and $\bfw \in \mathcal{W}$,  $\ggammanorm(\bfw, \bfz, \bgamma\muextraarg)$ is distributed as a $v$-sub-exponential random variable, conditioned upon $\bgamma$ being sampled from $\mathcal{D}$.
\end{assumption}

\begin{assumption}[Restriction on $\cW$]
     $\mathcal{W}$ is both convex and compact with a predefined diameter $D$. 
\end{assumption}
With these assumptions in place, we are now positioned to establish a bound on the discrepancy between the robust estimate and the expected estimate among benign clients.
\begin{theorem}[Robustness Error Bound]
    \label{thm:delta}
    Let $\mu \ge 0$, $\bfz \in \mathbb{S}^d, \epsilon>0$. 
    Then for any $\bfz_r \in \mathbb{S}^d$ for $r \in [k]$, under Assumptions \ref{assump:smooth}, \ref{assump:subexp},
    $\alpha \le \beta < \frac{1}{2} - \epsilon$, and with probability at least $1-\frac{4}{(1+2nm\hat{L}_\mu)^d(1+nm\hat{L}_\mu D)^d}$ \vspace{-.2cm}
    \begin{align*}
        \Big\| \frac{1}{k}\sum_{r=1}^k \hat{g}_{\beta}(\bfw, \bfz_r\muextraarg)\bfz_r  - \frac{1}{k}\sum_{r=1}^k \gpop(\bfw, \bfz_r\muextraarg)\Big\|_2 \le \Delta,
    \end{align*} \vspace{-.3cm}
    for
    \begin{align}
        \Delta \define & \sqrt{\log{(1\! +\! nm\hat{L}_\mu D)} \! + \! \log{(1\!+\!2nm\hat{L}_\mu)} \!+\! \frac{1}{d}\log{m}}  \nonumber    
                 \times \frac{v\sqrt{d}}{\epsilon} \Bigg( \frac{3\beta}{\sqrt{n}} \!+\! \frac{\sqrt{2}}{\sqrt{nm}} \Bigg) + \tilde{O}\Bigg(\frac{\beta}{n} \!+\! \frac{1}{mn}\Bigg). \label{def:delta}
    \end{align}
\end{theorem}
The proof of Theorem \ref{thm:delta} is found in Appendix \ref{proof:thm-delta}

\subsection{Convergence analysis}
To study the convergence properties of the proposed algorithm, we explore how the algorithm behaves in
scenarios with different values of the parameter $\mu$ under the assumption of strong convexity.

\begin{assumption}[Strong-Convexity]
    \label{assump:strong-convex}
    $F(\bfw)$ is $\lambda$-strongly convex.
\end{assumption}

Under this premise, we proceed to analyze the convergence characteristics of Algorithm \ref{alg:theoretical} for $\mu = 0$.

\begin{theorem}[Convergence, $\mu = 0$]
    \label{thm:convergence-mu-zero}
    Assume $\mu = 0, \epsilon>0$, set $\tau = \frac{d+k-1}{k} and $ $\eta = \frac{1}{\tau L_F}$ in Algorithm \ref{alg:theoretical},let $\mathcal{Z}^{t} = \left\{\bfz_1^t, \bfz_2^t, \dots, \bfz_k^t\right\}$ be the $k$ uniformly and independently sampled vectors from $\mathbb{S}^d$ and let $\mathcal{Z}^t = \bigcup_{i=0}^t Z^i$.
    Under Assumptions \ref{assump:smooth}, \ref{assump:subexp}, and \ref{assump:strong-convex}, with the constraint $\alpha \le \beta < \frac{1}{2} - \epsilon$, and a
    probability of at least $1-\frac{4}{(1+2nm\hat{L}_\mu)^d(1+nm\hat{L}_\mu D)^d}$, we have:
    \begin{align*} \mathbb{E}_{\mathcal{Z}^T} [ \|\bfw^T - \bfw^*\|_2] \le  \Big(1 - \frac{\lambda}{\tau(L_F + \lambda)}\Big)^T \|\bfw^0 - \bfw^*\|_2 + \frac{2}{\lambda}\Delta
    \end{align*}
    where $\bfw^T$ is the parameter vector at the $T$-th step %
    and $\Delta$ is the same as defined in \cref{thm:delta}.
\end{theorem}
The proof of this theorem is presented in Appendix \ref{proof:thm-convergence-mu-zero}.

To achieve an order-optimal error rate of $\tilde{O}(\frac{\beta}{\sqrt{n}} + \frac{1}{\sqrt{nm}})$, as suggested
by \cite{yinByzantineRobustDistributedLearning2021}, where the $\beta$ term represents the introduced error by the Byzantine behavior, the number of training steps $T$ should be at
least $\frac{\tau(L_F + \lambda)}{\lambda}\log{(\frac{\lambda}{2\Delta}\|\bfw^T - \bfw^*\|_2)}$.

Consistent with zero-order theory \cite{nesterovRandomGradientFreeMinimization2017,gao2018information}, our analysis reveals a linear
dependence of the convergence rate on the model dimension $d$, not observed in first-order methods. This dependence is encapsulated in the parameter $\tau$. As the number of sampled perturbation directions $k$ increases, our algorithm approximates the standard first-order rate of convergence.

Next, we examine the case where $\mu > 0$.

\begin{definition}[Smoothed Version of $F$]
    \label{def:smoothed-version}
    Let $F$ be a population loss, for any $\mu > 0$ and $\bfw \in \mathbb{R}^d$. The \emph{smoothed version} of $F$,  $F_\mu : \mathbb{R}^d \rightarrow \mathbb{R}$ is defined as 
    $ %
        F_\mu (\bfw) = \mathbb{E}_{\bfz \sim \mathbb{S}^d}[F(\bfw +\mu \bfz)].
    $ %
\end{definition}
\begin{assumption}[Local Minimum of $F_\mu$]
    \label{ass:local-minimum-smoothed} $\mathbf{w_\mu^*} = \arg\min_{\bfw \in \mathcal{W}} F_\mu(\bfw)$ is local minimum of $F_\mu$.
\end{assumption}

It can be shown (Lemma 4.1(b) of \cite{gao2018information}) 
that for any $\bfw$, $F(\bfw)$ and $F_\mu(\bfw)$ cannot differ more 
than $\frac{L\mu^2}{2}$, implying that the solutions $F(\bfw^*)$ and $F(\mathbf{w_\mu^*})$ can be made arbitrarily close by the choice of $\mu$.

\begin{theorem}[Convergence, $\mu> 0$]
    \label{thm:convergence-mu-positive}
    Assume $\mu > 0$,  set $\tau = \frac{2d + (k-1)\big(1+\sqrt{d}\big)}{k}$ and $\eta = \frac{1}{2 \tau L_F}$ in Algorithm \ref{alg:theoretical} and let $\mathcal{Z}^t$ be
    as in Theorem \ref{thm:convergence-mu-zero}. Under Assumptions \ref{assump:smooth}, \ref{assump:subexp}, and \ref{assump:strong-convex}, with $\alpha \le \beta < \frac{1}{2} - \epsilon$, and a probability of at least $1-\frac{4}{(1+2nm\hat{L}_\mu)^d(1+nm\hat{L}_\mu D)^d}$, we have:
    \begin{align*}
        \mathbb{E}_{\mathcal{Z}^T}  [ \|\bfw^T - \bfw_\mu^*\|_2] \le  \Big(1F - \frac{\lambda}{2\tau(L_F + \lambda)}\Big)^T  \|\bfw^0 - \bfw_\mu^*\|_2 \frac{2}{\lambda}\Delta + \frac{8\sqrt{2}\mu d L_F}{\lambda}\sqrt{1 + \frac{\tau}{4}},
    \end{align*}

    where $\bfw^T$ is the parameter vector at the $T$-th step of Algorithm \ref{alg:theoretical} and $\Delta$ is as defined in \cref{thm:delta}.
\end{theorem}
The proof of Theorem \ref{thm:convergence-mu-positive} is found in Appendix \ref{proof:thm-convergence-mu-positive}.

Similar to the results for $\mu = 0$, with $\mu > 0$, we observe the
same order-optimal error rate under appropriate conditions for $T$. 
Notably, for $k=1$, the
convergence rate scales linearly with the dimension $d$. However, unlike the case with $\mu = 0$, increasing $k$ does not entirely mitigate this linear dependency, leaving a residual term proportional to $\sqrt{d}$.

\section{Related Work}
\label{sec:Related Work}

\textbf{Zero-order Optimization.} In recent years, zero-order optimization has
significantly evolved, broadening its applicability
across various domains. This technique has been particularly instrumental in areas such
as black-box optimization \cite{ilyas2018black, andriushchenko2020square} and reinforcement learning \cite{salimansEvolutionStrategiesScalable2017, abdullah2019wasserstein} as gradient computations are not required.

A novel and notable application has emerged with the fine-tuning of LLMs
\cite{malladiFineTuningLanguageModels2023}, showcasing the versatility of zero-order optimization as a memory-efficient technique to allow the fine-tuning of billion-parameter models, while only using a fraction of the memory required by the first-order counterparts. The federated counterpart was studied in \cite{fang2022communication}.

\textbf{Communication Efficiency.} Communication efficiency through gradient
compression, can be mainly divided into two categories: quantization-based methods \cite{seide20141, bernstein2018signsgd,karimireddy2019error},
and sparsification-based methods \cite{vogels2019powersgd, m2021efficient}. None of those, however, are tailored for zero-order estimates, in the sense that they act on the $d$-dimensional gradients. 

Our work relies on compression on the gradient estimate by transmitting only the difference in perturbation losses, while taking advantage of an agreed randomness. We map it back to the work of \cite{salimansEvolutionStrategiesScalable2017}, passing through \cite{malladiFineTuningLanguageModels2023} and \cite{zelikmanJustOneByte2023}.

\textbf{Byzantine Resilience.}
To mitigate adversarial conditions in the learning process, recent research has seen a surge in the development of Byzantine resilient algorithms
\cite{blanchardMachineLearningAdversaries2017,yinByzantineRobustDistributedLearning2021,elkordy2022basil,xuByzantinerobustFederatedLearning2023}. 
Our work is closely related to \cite{yinByzantineRobustDistributedLearning2021}, which applies one-dimensional coordinate-wise statistical robustness techniques to the gradient information transmitted by the clients. In contrast, our countermeasure is applied for each perturbation, posing new theoretical challenges and bridging zero-order estimation and statistical robustness.

\textbf{Poisoning attacks.} From the wide range of attacks applied to Byzantine-resilient algorithms, we specially mention data poisoning \cite{fung2018mitigating,tolpegin2020data} and model manipulation \cite{fang2020local}. Both attacks enable Byzantine clients to effectively reduce the speed of convergence, increase the error rate, or completely disrupt the optimization result. 
Our attack model is related to \cite{fang2020local}, in which the Byzantine clients use the full information about the other clients' responses to maximize the deviation from the benign estimate in each optimization step.

\section{Conclusion}
In this paper, we have introduced \algname, a novel federated zero-order optimization scheme designed to withstand Byzantine behaviors. We demonstrated its effectiveness even under challenging conditions, which include a coordinated full-knowledge attack on non-IID data distributions. Our theoretical analysis underlines the robustness of \algname, illustrating the limited ability of Byzantine clients to significantly influence the learning process. In future work, \algname can be further enhanced by exploring advanced compression techniques, such as quantization, to optimize the transmission of values within \algname. 

The importance of integrating privacy-enhancing measures further opens up many potential directions including the incorporation of differential privacy, homomorphic encryption, and secure multi-party computation. The inherent efficiency of \algname yields promising compatibility with privacy-preserving techniques, potentially opening up new frontiers in secure, private and efficient federated learning.

\printbibliography

\appendix

\section{Experiments}
\subsection{Experimental Algorithm}

\label{app:algorthm}
In this appendix, we elaborate on the algorithm utilized for the experiments presented in Section \ref{sec:experiments}. Algorithm \ref{alg:practical} illustrates the optimization steps that enable \algname to attain both communication and memory efficiency. Notably, the key differences from Algorithm \ref{alg:theoretical} include:

\textbf{Perturbation Direction Sampling.} Contrary to the original method of sampling the perturbation directions $\bfz$ from the
unit sphere $\mathbb{S}^d$ and scaling the estimate by a factor of $d$, our practical approach, similar to \cite{salimansEvolutionStrategiesScalable2017, malladiFineTuningLanguageModels2023} involves sampling each coordinate
of $\bfz$ independently from a standard Gaussian distribution. This modification, while seemingly minor, has significant practical
implications. On one hand, this modification retains our theoretical guarantees --  the expected norm squared is identical to that of sampling over the unit sphere, and the normalized variance of the norm also decreases with increasing $d$. At the same time, from a practical point of view, it facilitates iterative sampling for each gradient coordinate, thereby considerably reducing the memory requirement in contrast to allocating the entire vector.

\textbf{Shared Randomness.} Since zero-order optimization requires clients and the federator to agree upon the perturbation directions, one approach would be that, each client receives the transmission
of perturbation directions from the federator. This would entail a significant downlink communication overhead. Instead, our optimized method relies on pseudorandom generation. 
A single seed value is transmitted from the federator to the clients at the start of the training loop. This enables every client to independently reconstruct all perturbation directions, dismissing the need for direct transmission
from the server. The principal advantage of this strategy lies in its significant reduction of communication
costs per training step—from communicating a $d$-dimensional vector to merely a single value per sample.

\textbf{In-Place Operations.} The standard procedure involves allocating memory for intermediate values
such as perturbation vectors and perturbed models. In our approach, again similar to \cite{malladiFineTuningLanguageModels2023}, these operations are executed in place.
This method not only conserves memory but also streamlines the computational process. By performing and subsequently
reversing these operations in place, we manage to maintain a low memory footprint throughout the training phase.
\newline

\begin{breakablealgorithm}
    \caption{\algname - Experimental Setup}
    \label{alg:practical}
    \begin{algorithmic}
        \STATE {\bfseries Input: } $\bfw^0$ is the initial model parameter vector, $\beta$ is the trimmed mean factor, $\eta$ is the learning rate, $\mu$ is the perturbation step, $k$ is the number of samples per estimate and $T$ is the total number of learning steps.
        \STATE $\blacktriangleright$ \textbf{Federator}
        \STATE $\;\;\;$ Distributes a random seed $s$ to each client
        \STATE $\;\;\;$ Distributes $\bfw_i^0 = \bfw^0$ to each client
        \FOR{$t=0$ \textbf{to} $T$}
        \FOR{$i = 1$ \textbf{to} $m$ \textbf{in parallel}}
        \STATE $\triangleright$ \textbf{client} $i$
        \STATE $\;\;\;$ \textbf{for} $r=1$ \textbf{to} $k$ \textbf{do}
        \STATE $\;\;\;\;\;\;$ $s' \leftarrow (s, t, r)$
        \STATE $\;\;\;\;\;\;$ $g_r^i \leftarrow $ $\begin{cases*} \text{\textsc{ZOGrad}(} F, \bfw_i^t, B, \mu, s' \text{)} & \text{if} $i\not\in\mathcal{B}$ \\ * & else\end{cases*}$
        \STATE $\;\;\;$ \textbf{end for}
        \STATE $\;\;\;$ Sends $\left\{g_r^i\right\}_{r=1}^k$ to federator
        \ENDFOR
        \STATE
        \STATE \textcolor{gray}{// Federator robust aggregation}
        \STATE $\blacktriangleright$ \textbf{Federator}
        \STATE $\;\;\;$ \textbf{for} $r=1$ \textbf{to} $k$ \textbf{do}
        \STATE $\;\;\;\;\;\;$ $\hat{g}_r \leftarrow trmean_\beta\left\{ \left\{g_r^i\right\}_{i=1}^m \right\}$
        \STATE $\;\;\;\;\;\;$ $\bfw^t \leftarrow$ \textsc{perturbParams}($\bfw^t, -\frac{\eta\hat{g}_r}{k} , s'$)
        \STATE $\;\;\;$ \textbf{end for}
        \STATE $\;\;\;$ Sends $\left\{\left\{\hat{g}_r\right\}_{r=1}^k\right\}$ to each client
        \STATE
        \STATE \textcolor{gray}{// clients synchronization}
        \FOR{$i = 1$ \textbf{to} $m$ \textbf{in parallel}}
        \STATE $\triangleright$ \textbf{client} $i$
        \STATE $\;\;\;$ \textbf{for} $r=1$ \textbf{to} $k$ \textbf{do}
        \STATE $\;\;\;\;\;\;$ $s' \leftarrow (s, t, r)$
        \STATE $\;\;\;\;\;\;$ $\bfw_i^t \leftarrow$ \textsc{perturbParams}($\bfw_i^t, -\frac{\eta\hat{g}_r}{k} , s'$)
        \STATE $\;\;\;$ \textbf{end for}
        \ENDFOR
        \ENDFOR
        \STATE
        \FUNCTION{\textsc{ZOGrad}($F, \bfw, B, \mu, s'$)}
        \STATE $\bfw \leftarrow$ \textsc{perturbParams}($\bfw, \mu, s'$)
        \STATE $l_+ \leftarrow$ $F(\bfw, B)$
        \STATE $\bfw \leftarrow$ \textsc{perturbParams}($\bfw, -2\mu, s'$)
        \STATE $l_- \leftarrow$ $F(\bfw, B)$
        \STATE $\bfw \leftarrow$ \textsc{perturbParams}($\bfw, \mu, s'$)$\;$ \textcolor{gray}{// Reset model state}
        \STATE Return $\frac{l_+ -L_-}{2\mu} $
        \ENDFUNCTION
        \STATE
        \FUNCTION{\textsc{perturbParams}($\bfw, \mu, s'$)}
        \STATE Set RNG seed as $s'$
        \FOR{\textbf{each} $\bfw^{(i)}$ \textbf{in} $\bfw$}
        \STATE $z \sim \mathcal{N}(0, 1)$
        \STATE $\bfw^{(i)} \gets \bfw^{(i)} + \mu z$$\;$ \textcolor{gray}{// Memory efficient}
        \ENDFOR
        \ENDFUNCTION
    \end{algorithmic}
\end{breakablealgorithm}

We provide here the Algorithm summarizing the Full-Knowledge attack.
\begin{algorithm}[b!]
    \caption{Full-Knowledge Byzantine Behavior}
    \label{alg:attack}
    \begin{algorithmic}
        \STATE {\bfseries Input: }$\bfw^t$, $\beta$,
        $\mu$, $\bfz_r^t$, $\ginorm(\bfw^t, \bfz_r^t)$ $\forall i\in[m]$.
        \FOR{ device $b \in \mathcal{B}$}
        \STATE $\hat{g}_{true} \leftarrow \frac{1}{m}\sum_{i=1}^{m} \ginorm(\bfw^t, \bfz_r^t)$
        \IF{$\hat{g}_{true} \ge 0$}
        \STATE Send the $\lfloor \beta m\rfloor$-th smallest value of $\left\{\ginorm(\bfw^t, \bfz_r^t): i \in [m]\setminus \mathcal{B}\right\}$
        \ELSE
        \STATE Send the $\lfloor \beta m\rfloor$-th largest value of $\left\{\ginorm(\bfw^t, \bfz_r^t): i \in [m]\setminus \mathcal{B}\right\}$

        \ENDIF

        \ENDFOR
    \end{algorithmic}
\end{algorithm}

\newpage

\subsection{Simulation Parameters and Hyperparameters}
\label{app:params}
\subsubsection{Logistic Regression on MNIST}
The simulation parameters and hyperparameters for all \algname experiments in Section \ref{sec:mnist-experiments} are found in Table \ref{tab:params-mnist}. Exceptionally, in the experiments of Table \ref{tab:m-comparison-state-of-the-art}, we use $m=40$, to reproduce the same settings as in \cite{yinByzantineRobustDistributedLearning2021}. All results in Section \ref{sec:mnist-experiments} are averaged across three random seeds.

\begin{table}[h]
\caption{Simulation Parameters and Hyperparameters for Section \ref{sec:mnist-experiments}}
We run all of our simulations in a single GPU setting using an Nvidia RTX 4090 GPU.

\label{tab:params-mnist}
\vskip 0.15in
\begin{center}
\begin{small}
\begin{sc}
\begin{tabular}{lcr}
\toprule
\multicolumn{2}{c}{MNIST Data Set} \\
\midrule
Global Train Samples & 60,000 \\
Number of Clients & 12  \\
Number of Byzantine Clients & 3   \\
$\beta$ & 0.25\\
Learning Rate & $10^{-2}$ \\
Client Batch Size & 64
{} \\
Learning Steps & 400 \\
\bottomrule
\end{tabular}
\end{sc}
\end{small}
\end{center}
\vskip -0.1in
\end{table}

\subsubsection{Prompt-Based Fine-Tuning}
The simulation parameters and hyperparameters for all \algname experiments in Section \ref{sec:experiments-ft} are found in Table \ref{tab:params-ft}. For Table \ref{tab:main-comparison} and \ref{fig:sst-train-loss}, results are averaged across three different random seeds.
\begin{table}[h]
\caption{Simulation Parameters and Hyperparameters for Section \ref{sec:experiments-ft}}
\label{tab:params-ft}
\vskip 0.15in
\begin{center}
\begin{small}
\begin{sc}
\begin{tabular}{lcccr}
\toprule
& SST-2 & SNLI & TREC \\
\midrule
Global Train Samples & \multicolumn{3}{c}{\xrfill[3pt]{0.5pt} $\;$512 \xrfill[3pt]{0.5pt}} \\
Number of Clients & 8 & 12 & 12   \\
Number of Byzantine Clients & 2 & 3 & 3   \\
$\beta$ & \multicolumn{3}{c}{\xrfill[3pt]{0.5pt} $\;$0.25 \xrfill[3pt]{0.5pt}} \\
Learning Rate & \multicolumn{3}{c}{\xrfill[3pt]{0.5pt} $\;10^{-6}$ \xrfill[3pt]{0.5pt}} \\
Client Batch Size & \multicolumn{3}{c}{
 \xrfill[3pt]{0.5pt} $\;$64 \xrfill[3pt]{0.5pt}
} \\
Learning Steps & 20,000 & 20,000   & 40,000 \\
\bottomrule
\end{tabular}
\end{sc}
\end{small}
\end{center}
\vskip -0.1in
\end{table}

\subsection{Further Loss Curves for TREC and SNLI}
\label{app:roberta}
In this section, we present two additional loss curves for the TREC and SNLI experiments. Both curves are shown in Figure \ref{fig:trec-snli-train-loss}. For both datasets, we observe a convergence reduction for the Non-Byzantine setting. For SNLI, the reduction appears greater, which aligns with the results in Table \ref{tab:main-comparison}

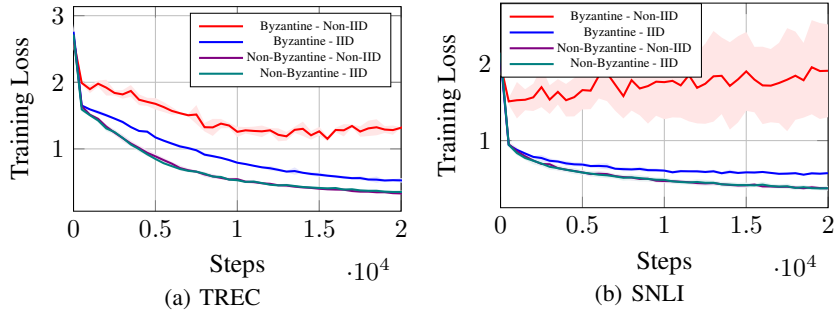
\begin{figure}[h]
    \centering

    \subfigure[TREC]{\begin{minipage}{0.4\textwidth}\resizebox{\textwidth}{!}{\begin{tikzpicture}
                \begin{axis}[
                        xlabel={Steps},
                        ylabel={Training Loss},
                        legend pos=north east,
                        grid=both,
                        xmin=0, xmax=20000,
                        scale only axis,
                        height=0.50\textwidth,
                        width=0.8\textwidth,
                        legend style={nodes={scale=0.5, transform shape}}, 
                    ]

                    \addplot [red, thick] table[col sep=comma, x expr=10*\thisrow{Epoch}, y=Loss] {data/trec_comp/TREC_3_non_iid_trimmed_mean_tm_full_knowledge_0.25_12_1.csv};
                    \addlegendentry{Byzantine - Non-IID}

                    \addplot [blue, thick] table[col sep=comma, x expr=10*\thisrow{Epoch}, y=Loss] {data/trec_comp/TREC_3_iid_trimmed_mean_tm_full_knowledge_0.25_12_1.csv};
                    \addlegendentry{Byzantine - IID}

                    \addplot [violet, thick] table[col sep=comma, x expr=10*\thisrow{Epoch}, y=Loss] {data/trec_comp/TREC_0_non_iid_naive_tm_full_knowledge_0.25_12_1.csv};
                    \addlegendentry{Non-Byzantine - Non-IID}

                    \addplot [teal, thick] table[col sep=comma, x expr=10*\thisrow{Epoch}, y=Loss] {data/trec_comp/TREC_0_iid_naive_tm_full_knowledge_0.25_12_1.csv};
                    \addlegendentry{Non-Byzantine - IID}

                    \addplot [name path=upper,draw=none] table[col sep=comma, x expr=10*\thisrow{Epoch}, y expr=\thisrow{Loss}+\thisrow{StdDev}] {data/trec_comp/TREC_0_iid_naive_tm_full_knowledge_0.25_12_1.csv};
                    \addplot [name path=lower,draw=none] table[col sep=comma, x expr=10*\thisrow{Epoch}, y expr=\thisrow{Loss}-\thisrow{StdDev}] {data/trec_comp/TREC_0_iid_naive_tm_full_knowledge_0.25_12_1.csv};
                    \addplot [fill=teal!10] fill between[of=upper and lower];

                    \addplot [name path=upper,draw=none] table[col sep=comma, x expr=10*\thisrow{Epoch}, y expr=\thisrow{Loss}+\thisrow{StdDev}] {data/trec_comp/TREC_3_iid_trimmed_mean_tm_full_knowledge_0.25_12_1.csv};
                    \addplot [name path=lower,draw=none] table[col sep=comma, x expr=10*\thisrow{Epoch}, y expr=\thisrow{Loss}-\thisrow{StdDev}] {data/trec_comp/TREC_3_iid_trimmed_mean_tm_full_knowledge_0.25_12_1.csv};
                    \addplot [fill=blue!10] fill between[of=upper and lower];

                    \addplot [name path=upper,draw=none] table[col sep=comma, x expr=10*\thisrow{Epoch}, y expr=\thisrow{Loss}+\thisrow{StdDev}] {data/trec_comp/TREC_0_non_iid_naive_tm_full_knowledge_0.25_12_1.csv};
                    \addplot [name path=lower,draw=none] table[col sep=comma, x expr=10*\thisrow{Epoch}, y expr=\thisrow{Loss}-\thisrow{StdDev}] {data/trec_comp/TREC_0_non_iid_naive_tm_full_knowledge_0.25_12_1.csv};
                    \addplot [fill=violet!10] fill between[of=upper and lower];

                    \addplot [name path=upper,draw=none] table[col sep=comma, x expr=10*\thisrow{Epoch}, y expr=\thisrow{Loss}+\thisrow{StdDev}] {data/trec_comp/TREC_3_non_iid_trimmed_mean_tm_full_knowledge_0.25_12_1.csv};
                    \addplot [name path=lower,draw=none] table[col sep=comma, x expr=10*\thisrow{Epoch}, y expr=\thisrow{Loss}-\thisrow{StdDev}] {data/trec_comp/TREC_3_non_iid_trimmed_mean_tm_full_knowledge_0.25_12_1.csv};
                    \addplot [fill=red!10] fill between[of=upper and lower];
                \end{axis}
            \end{tikzpicture}}\end{minipage}} 
    \subfigure[SNLI]{\begin{minipage}{0.4\textwidth}\resizebox{\textwidth}{!}{ \begin{tikzpicture}
                \begin{axis}[
                        xlabel={Steps},
                        ylabel={Training Loss},
                        grid=both,
                        xmin=0, xmax=20000,
                        scale only axis,
                        height=0.50\textwidth,
                        width=0.8\textwidth,
                        legend style={nodes={scale=0.5, transform shape}, at = {(0,1)}, anchor = {north west}, fill opacity = 0.5, text opacity = 1}, 
                    ]

                    \addplot [red, thick] table[col sep=comma, x expr=10*\thisrow{Epoch}, y=Loss] {data/snli_comp/SNLI_3_non_iid_trimmed_mean_tm_full_knowledge_0.25_12_1.csv};
                    \addlegendentry{Byzantine - Non-IID}

                    \addplot [blue, thick] table[col sep=comma, x expr=10*\thisrow{Epoch}, y=Loss] {data/snli_comp/SNLI_3_iid_trimmed_mean_tm_full_knowledge_0.25_12_1.csv};
                    \addlegendentry{Byzantine - IID}

                    \addplot [violet, thick] table[col sep=comma, x expr=10*\thisrow{Epoch}, y=Loss] {data/snli_comp/SNLI_0_non_iid_naive_tm_full_knowledge_0.25_12_1.csv};
                    \addlegendentry{Non-Byzantine - Non-IID}

                    \addplot [teal, thick] table[col sep=comma, x expr=10*\thisrow{Epoch}, y=Loss] {data/snli_comp/SNLI_0_iid_naive_tm_full_knowledge_0.25_12_1.csv};
                    \addlegendentry{Non-Byzantine - IID}

                    \addplot [name path=upper,draw=none] table[col sep=comma, x expr=10*\thisrow{Epoch}, y expr=\thisrow{Loss}+\thisrow{StdDev}] {data/snli_comp/SNLI_3_non_iid_trimmed_mean_tm_full_knowledge_0.25_12_1.csv};
                    \addplot [name path=lower,draw=none] table[col sep=comma, x expr=10*\thisrow{Epoch}, y expr=\thisrow{Loss}-\thisrow{StdDev}] {data/snli_comp/SNLI_3_non_iid_trimmed_mean_tm_full_knowledge_0.25_12_1.csv};
                    \addplot [fill=red!10] fill between[of=upper and lower];

                    \addplot [name path=upper,draw=none] table[col sep=comma, x expr=10*\thisrow{Epoch}, y expr=\thisrow{Loss}+\thisrow{StdDev}] {data/snli_comp/SNLI_3_iid_trimmed_mean_tm_full_knowledge_0.25_12_1.csv};
                    \addplot [name path=lower,draw=none] table[col sep=comma, x expr=10*\thisrow{Epoch}, y expr=\thisrow{Loss}-\thisrow{StdDev}] {data/snli_comp/SNLI_3_iid_trimmed_mean_tm_full_knowledge_0.25_12_1.csv};
                    \addplot [fill=blue!10] fill between[of=upper and lower];

                    \addplot [name path=upper,draw=none] table[col sep=comma, x expr=10*\thisrow{Epoch}, y expr=\thisrow{Loss}+\thisrow{StdDev}] {data/snli_comp/SNLI_0_non_iid_naive_tm_full_knowledge_0.25_12_1.csv};
                    \addplot [name path=lower,draw=none] table[col sep=comma, x expr=10*\thisrow{Epoch}, y expr=\thisrow{Loss}-\thisrow{StdDev}] {data/snli_comp/SNLI_0_non_iid_naive_tm_full_knowledge_0.25_12_1.csv};
                    \addplot [fill=violet!10] fill between[of=upper and lower];

                    \addplot [name path=upper,draw=none] table[col sep=comma, x expr=10*\thisrow{Epoch}, y expr=\thisrow{Loss}+\thisrow{StdDev}] {data/snli_comp/SNLI_0_iid_naive_tm_full_knowledge_0.25_12_1.csv};
                    \addplot [name path=lower,draw=none] table[col sep=comma, x expr=10*\thisrow{Epoch}, y expr=\thisrow{Loss}-\thisrow{StdDev}] {data/snli_comp/SNLI_0_iid_naive_tm_full_knowledge_0.25_12_1.csv};
                    \addplot [fill=teal!10] fill between[of=upper and lower];

                \end{axis}
            \end{tikzpicture}}\end{minipage}}

\caption{Contrasting Byzantine and Non-Byzantine Scenarios Across Diverse Data Distributions with RoBERTa-large on TREC and SNLI: This figure compares the performance of \algname using a RoBERTa-large model on both TREC and SNLI datasets. Non-Byzantine behavior stands for \algname with no Byzantine clients nor robust aggregation.}  
\label{fig:trec-snli-train-loss}
\end{figure}
\subsection{Local Epochs}

\label{sec:local-epochs}
In our experimental framework, we extended our investigation to evaluate the performance of \algname under the context of local epochs. This involved adapting our practical algorithm (Algorithm \ref{alg:practical}) to include local epochs, incorporating memory-efficient operations, as detailed in Algorithm \ref{alg:practical-le}.

Our findings, as summarized in Table \ref{tab:local-epochs-sst2}, reveal an intriguing parallel to the impact of the variable $k$. Specifically, modifying the number of local epochs appears to exert a comparable influence. A notable insight from this experiment is the relative stability of the final test accuracy, despite variations in batching these training epochs. However, it's important to acknowledge that the efficacy of such a technique can be highly dependent on the specific problem at hand, suggesting a need for cautious interpretation and application in different contexts. %

\begin{table}[h]
    \caption{Assessing the Impact of Varying Local Epochs in SST-2 Using RoBERTa-large: This table provides a comparative analysis of how different settings for local epochs affect performance on the SST-2 dataset using the RoBERTa-large model on a Byzantine and non-IID setting.}
    \label{tab:local-epochs-sst2}
    \vskip 0.15in
    \begin{center}
        \begin{small}
            \begin{sc}
                \begin{tabular}{lccr}
                    \toprule
                    \multirow{2}{*}{Local Epochs} & Communication & \multirow{2}{*}{Test Accuracy} \\
                                                  & Rounds        &                                \\
                    \midrule
                    1                             & 20,000        & 92.7   $\pm$ $0.4$                      \\
                    5                             & 4,000         & 92.6   $\pm$ $0.3$                      \\
                    50                            & 400           & 92.8   $\pm$ $0.5$                    \\
                    100                           & 200           & 92.4     $\pm$ $0.2$                 \\

                    \bottomrule
                \end{tabular}
            \end{sc}
        \end{small}
    \end{center}
    \vskip -0.1in
\end{table}
\newpage
\begin{breakablealgorithm}
    \caption{\algname - Experimental Setup with Local Epochs}
    \label{alg:practical-le}
    \begin{algorithmic}
        \STATE {\bfseries Input: } $\bfw^0$ is the initial model parameter vector, $\beta$ is the trimmed mean factor, $\eta$ is the learning rate, $\mu$ is the perturbation step, $k$ is the number of samples per estimate, $E$ is the number of local epochs and $T$ is the total number of learning steps.
        \STATE $\blacktriangleright$ \textbf{Federator}
        \STATE $\;\;\;$ Distributes a random seed $s$ to each client
        \STATE $\;\;\;$ Distributes $\bfw_i = \bfw^0$ to each client
                \STATE $\;\;\;$ Sets $\bfw \leftarrow \bfw^0$
        \FOR{$t=1$ \textbf{to} $T$}
        \FOR{$i = 1$ \textbf{to} $m$ \textbf{in parallel}}
        \STATE $\triangleright$ \textbf{client} $i$
        \STATE $\;\;\;$ \textbf{for} $e=1$ \textbf{to} $E$ \textbf{do}
                \STATE $\;\;\;\;\;\;$ \textcolor{gray}{// Sample $k$ grad estimates}

        \STATE $\;\;\;\;\;\;$ \textbf{for} $r=1$ \textbf{to} $k$ \textbf{do}
        \STATE $\;\;\;\;\;\;\;\;\;$ $s' \leftarrow (s, t, r, e)$
        \STATE $\;\;\;\;\;\;\;\;\;$ $g_{r, e}^{i,t} \leftarrow $ $\begin{cases*} \text{\textsc{ZOGrad}(} F, \bfw_i, B, \mu, s' \text{)} & \text{if} $i\not\in\mathcal{B}$ \\ * & else\end{cases*}$
        \STATE $\;\;\;\;\;\;$ \textbf{end for}
        \STATE $\;\;\;\;\;\;$ \textcolor{gray}{// Apply local epoch learning step}
        \STATE $\;\;\;\;\;\;$ \textbf{for} $r=1$ \textbf{to} $k$ \textbf{do}
               \STATE $\;\;\;\;\;\;\;\;\;$ $s' \leftarrow (s, t, r, e)$
        \STATE $\;\;\;\;\;\;\;\;\;$ $\bfw_i \leftarrow$ \textsc{perturbParams}($\bfw_i, -\frac{\eta g_{r, e}^{i,t} }{k} , s'$)
        \STATE $\;\;\;\;\;\;$ \textbf{end for}
        \STATE $\;\;\;$ \textbf{end for}

        \STATE $\;\;\;$ \textcolor{gray}{// Reset model to start of local epochs (alternatively, can store initial model, with extra memory cost)}
                \STATE $\;\;\;$ \textbf{for} $e=1$ \textbf{to} $E$ \textbf{do}
        \STATE $\;\;\;\;\;\;$ \textbf{for} $r=1$ \textbf{to} $k$ \textbf{do}
               \STATE $\;\;\;\;\;\;\;\;\;$ $s' \leftarrow (s, t, r, e)$
        \STATE $\;\;\;\;\;\;\;\;\;$ $\bfw_i \leftarrow$ \textsc{perturbParams}($\bfw_i, \frac{\eta g_{r, e}^{i,t} }{k} , s'$)
        \STATE $\;\;\;\;\;\;$ \textbf{end for}
        \STATE $\;\;\;$ \textbf{end for}
        
        \STATE $\;\;\;$ Sends $\left\{\left\{g_{r,e}^{i,t}\right\}_{r=1}^k\right\}_{e=1}^E$ to federator
        \ENDFOR
        \STATE
        \STATE \textcolor{gray}{// Federator robust aggregation}
        \STATE $\blacktriangleright$ \textbf{Federator}
        \STATE $\;\;\;$ \textbf{for} $e=1$ \textbf{to} $E$ \textbf{do}
        \STATE $\;\;\;\;\;\;$ \textbf{for} $r=1$ \textbf{to} $k$ \textbf{do}
        
        \STATE $\;\;\;\;\;\;\;\;\;$ $\hat{g}_{r,e}^t \leftarrow \Trmean_\beta\left\{ \left\{g_{r,e}^i\right\}_{i=1}^m \right\}$
        \STATE $\;\;\;\;\;\;\;\;\;$ $s' \leftarrow (s, t, r, e)$
        \STATE $\;\;\;\;\;\;\;\;\;$ $\bfw \leftarrow$ \textsc{perturbParams}($\bfw, -\frac{\eta\hat{g}_{r,e}^t}{k} , s'$)
        \STATE $\;\;\;\;\;\;$ \textbf{end for}
        \STATE $\;\;\;$ \textbf{end for}
        \STATE $\;\;\;$ Sends $\left\{\left\{\hat{g}_{r,e}^t\right\}_{r=1}^k\right\}_{e=1}^E$ to each client
        \STATE
        \STATE \textcolor{gray}{// clients synchronization}
        \FOR{$i = 1$ \textbf{to} $m$ \textbf{in parallel}}
        \STATE $\triangleright$ \textbf{client} $i$
         \STATE $\;\;\;$ \textbf{for} $e=1$ \textbf{to} $E$ \textbf{do}
        \STATE $\;\;\;\;\;\;$ \textbf{for} $r=1$ \textbf{to} $k$ \textbf{do}
        \STATE $\;\;\;\;\;\;\;\;\;$ $s' \leftarrow (s, t, r, e)$
        \STATE $\;\;\;\;\;\;\;\;\;$ $\bfw_i \leftarrow$ \textsc{perturbParams}($\bfw_i, -\frac{\eta\hat{g}_{r,e}^t}{k} , s'$)
        \STATE $\;\;\;\;\;\;$ \textbf{end for}
        \STATE $\;\;\;$ \textbf{end for}
        \ENDFOR
        \ENDFOR
        \STATE
        \FUNCTION{\textsc{ZOGrad}($F, \bfw, B, \mu, s'$)}
        \STATE $\bfw \leftarrow$ \textsc{perturbParams}($\bfw, \mu, s'$)
        \STATE $l_+ \leftarrow$ $F(\bfw, B)$
        \STATE $\bfw \leftarrow$ \textsc{perturbParams}($\bfw, -2\mu, s'$)
        \STATE $l_- \leftarrow$ $F(\bfw, B)$
        \STATE $\bfw \leftarrow$ \textsc{perturbParams}($\bfw, \mu, s'$)$\;$ \textcolor{gray}{// Reset model state}
        \STATE Return $\frac{l_+ -L_-}{2\mu} $
        \ENDFUNCTION
        \STATE
        \FUNCTION{\textsc{perturbParams}($\bfw, \mu, s'$)}
        \STATE Set RNG seed as $s'$
        \FOR{\textbf{each} $\bfw^{(i)}$ \textbf{in} $\bfw$}
        \STATE $z \sim \mathcal{N}(0, 1)$
        \STATE $\bfw^{(i)} \gets \bfw^{(i)} + \mu z$$\;$ \textcolor{gray}{// Memory efficient}
        \ENDFOR
        \ENDFUNCTION
    \end{algorithmic}
\end{breakablealgorithm}

\newpage

\section{Proofs}
\label{app:proofs}
In this appendix we derive the proofs for the Theorems in Section \ref{sec:theory}.

\subsection{Proof of Theorem \ref{thm:delta}}
\label{proof:thm-delta}
In proving Theorem \ref{thm:delta}, we begin by invoking a lemma from \cite{yinByzantineRobustDistributedLearning2021}. This lemma provides a probabilistic upper bound for the maximum deviation in a one-dimensional robust mean estimation problem of a random variable, within the context of our client setup described in Section \ref{sec:prob-setting}.

\begin{lemma}
    \label{lemma:rob-est}
    (Lemma 3, \cite{yinByzantineRobustDistributedLearning2021}) Let $x$ be a v-sub-exponential random variable with mean $\mu_x$. For all $i \in [m]$ and $j\in [n]$ Let $x^{i,j}$ is the $j$-th sample of $x$ in client $i$ if the client is benign or arbitrary adversarial data otherwise, and let $\bar{x}^i = \frac{1}{n}\sum_{j=1}^n x^{i,j}$ . Then, for any $t>0$, Byzantine fraction $0le\alpha <\frac{1}{2}$ and trimmed mean factor $\beta$.

    \begin{align*}
        P\left\{|\frac{1}{(1-\alpha)m}\sum_{i\in [m] \setminus \mathcal{B}} \bar{x}^i - \mu_x| \ge t\right\} \le 2\exp{\left\{-(1-\alpha)mn\min{\left\{\frac{t}{2v}, \frac{t^2}{2v^2}\right\}}\right\}},
    \end{align*}
    for any $s > 0$
    \begin{align*}
        P\left\{\max_{i\in [m] \setminus \mathcal{B}} \left\{|\bar{x}^i - \mu_x|\right\} \ge s\right\} \le 2(1-\alpha)m\exp{\left\{-n\min{\left\{\frac{s}{2v}, \frac{s^2}{2v^2}\right\}}\right\}},
    \end{align*}
    and when $\beta \ge \alpha$, $|\frac{1}{(1-\alpha)m}\sum_{i\in [m] \setminus \mathcal{B}} \bar{x}^i - \mu_x| \le t$, and $\max_{i\in [m] \setminus \mathcal{B}} \left\{|\bar{x}^i - \mu_x|\right\} \le s$
    \begin{align*}
        |\Trmean_\beta\left\{\bar{x}^i: i \in [m]\right\} - \mu_x| \le \frac{t+3\beta s}{1-2\beta}.
    \end{align*}

\end{lemma}

With Lemma \ref{lemma:rob-est} as our foundation, we proceed to prove Theorem \ref{thm:delta} by applying it to the norm of the client estimate $\ginorm(\bfw, \bfz\muextraarg)$. For $s,t,\mu \ge 0$ and $\bfz \in \mathbb{S}^d$, we can state that with no smaller probability than

\begin{align*}
    1 - 2\exp{\left\{-(1-\alpha)mn\min{\left\{\frac{t}{2v}, \frac{t^2}{2v^2}\right\}}\right\}} - 2(1-\alpha)m\exp{\left\{-n\min{\left\{\frac{s}{2v}, \frac{s^2}{2v^2}\right\}}\right\}},
\end{align*}
the trimmed mean of client estimates is bounded as
\begin{align*}
    |\hat{g}_{\beta}(\bfw, \bfz\muextraarg) - \gpopnorm(\bfw, \bfz\muextraarg)| = |\Trmean_\beta\left\{\ginorm(\bfw, \bfz\muextraarg): i \in [m]\right\} - \gpopnorm(\bfw, \bfz\muextraarg)| \le \frac{t+3\beta s}{1-2\beta}.
\end{align*}

To extend this result to all $\bfw \in \mathcal{W}$ and any $k$ samples of vectors $\bfz_r$, we follow the methodology
outlined in \cite{yinByzantineRobustDistributedLearning2021}, utilizing a covering set argument on both $\mathcal{W}$ and $\mathbb{S}^d$. We define
$W_\delta = \left\{ \bfw^1, \bfw^2, \dots, \bfw^{N_\delta}\right\}$ and $Z_\delta = \left\{ \bfz^1, \bfz^2, \dots, \bfz^{M_\delta}\right\}$, where $N_\delta, M_\delta \in \mathbb{N}$,
such that for all $\bfw \in \mathcal{W}$ and $\bfz\in\mathbb{S}^d$, there exists an
$l \in [N_\delta]$ and $q \in [M_\delta]$ such that $\|\bfw^l - \bfw\|_2 \le \delta$ and $\|\bfz^q - \bfz\|_2 \le \delta$.
According to \cite{vershynin2010introduction}, it is always possible to find such sets satisfying $N_\delta \le (1+\frac{D}{\delta})^d$ and $M_\delta \le (1+\frac{2}{\delta})^d$.

Applying the union bound, with probabilities at least $1 - 2M_\delta N_\delta\exp{\left\{-(1-\alpha)mn\min{\left\{\frac{t}{2v}, \frac{t^2}{2v^2}\right\}}\right\}}$ and $1 - 2M_\delta N_\delta(1-\alpha)m\exp{\left\{-n\min{\left\{\frac{s}{2v}, \frac{s^2}{2v^2}\right\}}\right\}}$, we can ensure
that $|\frac{1}{(1-\alpha)m}\sum_{i\in [m] \setminus \mathcal{B}} \ginorm(\bfw^l, \bfz^q\muextraarg) - \gpopnorm(\bfw^l, \bfz^q\muextraarg)| \le t$ and
$\max_{i\in [m] \setminus \mathcal{B}} \left\{|\ginorm(\bfw^l, \bfz^q\muextraarg) - \gpopnorm(\bfw^l, \bfz^q\muextraarg)|\right\} \le s$, respectively,
for all $\bfw^l \in W_\delta$ and $\bfz^q \in Z_\delta$.

In the event of these joint conditions being met, we have

\begin{align*}
    | \hat{g}_{\beta}(\bfw^l, \bfz^q\muextraarg)  - \bar{g}(\bfw^l, \bfz^q\muextraarg)| \le \frac{t+3\beta s}{1-2\beta}
\end{align*}
for all $\bfw^l \in W_\delta$ and $\bfz^q \in Z_\delta$. 

Now, using both the $L_{w,\mu}$-Lipschitz property of $\ggammanorm(\cdot,\bfz, \bgamma\muextraarg)$ and the $L_{z,\mu}$-Lipschitz property of $\ggammanorm(\bfw, \cdot,\bgamma\muextraarg)$ for any data
sample $\bgamma$ from Assumption \ref{assump:smooth}, and the fact that for all $\bfw \in \mathcal{W}$ and $\bfz\in\mathbb{S}^d$, there exists $\bfw^l \in W_\delta$ and $\bfz^q \in Z_\delta$ such
that $\|\bfw^l - \bfw\|_2 \le \delta$ and $\|\bfz^q - \bfz\|_2 \le \delta$,
we have that for any benign client $i$, $|\ginorm(\bfw^l, \bfz^q \muextraarg) - \ginorm(\bfw, \bfz\muextraarg)| \le (L_{w,\mu}+L_{z,\mu})\delta$ and
$|\gpopnorm(\bfw^l, \bfz^q \muextraarg) - \gpopnorm(\bfw, \bfz\muextraarg)| \le (L_{w,\mu}+L_{z,\mu})\delta$.

So first, in the event of $|\frac{1}{(1-\alpha)m}\sum_{i\in [m] \setminus \mathcal{B}} \ginorm(\bfw^l, \bfz^q\muextraarg) - \gpopnorm(\bfw^l, \bfz^q\muextraarg)| \le t$ for all $\bfw^l \in W_\delta$ and $\bfz \in Z_\delta$  and by making use of the triangle inequality, we have for all $\bfw \in \mathcal{W}$ and $\bfz \in \mathbb{S}^d$

\begin{align*}
    \left|\frac{1}{(1-\alpha)m}\sum_{i\in [m] \setminus \mathcal{B}} \ginorm(\bfw, \bfz\muextraarg) - \gpopnorm(\bfw, \bfz\muextraarg)\right|  &\le \left|\frac{1}{(1-\alpha)m}\sum_{i\in [m] \setminus \mathcal{B}} \ginorm(\bfw^l, \bfz^q\muextraarg) - \gpopnorm(\bfw^l, \bfz^q\muextraarg)\right| \\ & \quad + \frac{1}{(1-\alpha)m}\sum_{i\in [m] \setminus \mathcal{B}} \left| \ginorm(\bfw^l, \bfz^q \muextraarg) - \ginorm(\bfw, \bfz\muextraarg) \right| + \left|\gpopnorm(\bfw^l, \bfz^q \muextraarg) - \gpopnorm(\bfw, \bfz\muextraarg) \right|\\ & \le t + 2 (L_{w,\mu}+L_{z,\mu})\delta,
\end{align*}

and similarly, having
$\max_{i\in [m] \setminus \mathcal{B}} \left\{|\ginorm(\bfw^l, \bfz^q\muextraarg) - \gpopnorm(\bfw^l, \bfz^q\muextraarg)|\right\} \le s$ for all $\bfw^l \in W_\delta$ and $\bfz \in Z_\delta$ implies that 
\begin{align*}
    \max_{i\in [m] \setminus \mathcal{B}} \left\{|\ginorm(\bfw, \bfz\muextraarg) - \gpopnorm(\bfw, \bfz\muextraarg)|\right\} \le s +2(L_{w,\mu}+L_{z,\mu})\delta
\end{align*}
further implying
\begin{align*}
    \|\hat{g}_{\beta}(\bfw, \bfz\muextraarg)\bfz - \gpop(\bfw, \bfz\muextraarg)\|_2 = |\Trmean_\beta\left\{\ginorm(\bfw, \bfz\muextraarg): i \in [m]\right\} - \gpopnorm(\bfw, \bfz\muextraarg)| \le \frac{t+3\beta s}{1-2\beta} + \frac{2(1+3\beta)}{1-2\beta}\delta(L_{w,\mu}+L_{z,\mu})
\end{align*}
which leads us, by applying the traingle inequality, to the final bound
\begin{align*}
    \Big\| \frac{1}{k}\sum_{r=1}^k \hat{g}_{\beta}(\bfw, \bfz_r\muextraarg)\bfz_r  - \frac{1}{k}\sum_{r=1}^k \gpop(\bfw, \bfz_r\muextraarg)\Big\|_2 \le \frac{t+3\beta s}{1-2\beta} + \frac{2(1+3\beta)}{1-2\beta}\delta \hat{L}_\mu
\end{align*}
for $\hat{L}_\mu = L_{w,\mu}+L_{z,\mu}$.

Choosing $\delta = \frac{1}{nm\hat{L}_\mu}$,
\begin{align*}
    t = v\max \left\{ \frac{8d}{nm}\log{[(1 + nm\hat{L}_\mu D)(1 + 2nm\hat{L}_\mu)]}, \sqrt{\frac{8d}{nm}\log{[(1 + nm\hat{L}_\mu D)(1 + 2nm\hat{L}_\mu)]} } \right\},
\end{align*}
and
\begin{align*}
    s = v\max \left\{ \frac{4}{n}(d \log{[(1 + nm\hat{L}_\mu D)(1 + 2nm\hat{L}_\mu)]} + \log{m}), \sqrt{\frac{4}{n}(d \log{[(1 + nm\hat{L}_\mu D)(1 + 2nm\hat{L}_\mu)]} + \log{m})} \right\}
\end{align*}
completes the proof. $ \square $

\subsection{Proof of Theorem \ref{thm:convergence-mu-zero}}
\label{proof:thm-convergence-mu-zero}
In order to prove Theorem \ref{thm:convergence-mu-zero} we instantiate several lemmas on the probabilistic behavior of the function $\gpop(\bfw, \bfz\muextraarg)$ for the case $\mu = 0$.

\begin{lemma}
    \label{lemma:exp-matrix}
    (Lemma 7.3 (b) \cite{gao2018information}) $\mathbb{E}_{\bfz \sim \mathbb{S}^d}[\bfz\bfz^\intercal] = \frac{1}{d}I_d$, where $I_d$ is the $d$-dimensional identity matrix.
\end{lemma}
\begin{lemma}
The following holds.    \label{lemma:cross-exp}
    \begin{align*}
        \mathbb{E}_{\bfz_1,\bfz_2  \sim \mathbb{S}^d}[\bfz_1\bfz_1^\intercal \bfz_2\bfz_2^\intercal] = \begin{cases}
                                                                                                                              \frac{1}{d}I_d,   & \text{if } \bfz_1 = \bfz_2,                                    \\
                                                                                                                              \frac{1}{d^2}I_d, & \text{if } \bfz_1 \text{ and } \bfz_2 \text{ are independent}.
                                                                                                                          \end{cases}
    \end{align*}
\end{lemma}
\begin{proof}
    In the case $\bfz_1 = \bfz_2$,
    \begin{align*}
        \mathbb{E}_{\bfz_1,\bfz_2  \sim \mathbb{S}^d}[\bfz_1\bfz_1^\intercal\bfz_2\bfz_2^\intercal] & = \mathbb{E}_{\bfz_1 \sim \mathbb{S}^d}[\bfz_1\bfz_1^\intercal\bfz_1\bfz_1^\intercal] \\
                                                                                                                        & = \mathbb{E}_{\bfz_1 \sim \mathbb{S}^d}[\bfz_1\bfz_1^\intercal]
    \end{align*}
    and the result is given by Lemma \ref{lemma:exp-matrix}.

    In the case that $\bfz_1 $ and $ \bfz_2$ are independent, 
    \begin{align*}
        \mathbb{E}_{\bfz_1,\bfz_2  \sim \mathbb{S}^d}[\bfz_1\bfz_1^\intercal\bfz_2\bfz_2^\intercal] & = \mathbb{E}_{\bfz_1 \sim \mathbb{S}^d}[\bfz_1\bfz_1^\intercal] \mathbb{E}_{\bfz_2 \sim \mathbb{S}^d}[\bfz_2\bfz_2^\intercal]
    \end{align*}
    And again, by Lemma \ref{lemma:exp-matrix}, the result is yielded.
\end{proof}

\begin{lemma}
    \label{lemma:expected-mu-zero}
    Let $\bfw \in \mathbb{R}^d$, %
    \begin{align*}
        \mathbb{E}_{\bfz \sim \mathbb{S}^d}[\gpop(\bfw, \bfz\zeroextraarg)] = \nabla F(\bfw).
    \end{align*}
\end{lemma}
\begin{proof}
    \begin{align*}
        \mathbb{E}_{\bfz \sim \mathbb{S}^d}[\gpop(\bfw, \bfz\zeroextraarg)] & = \mathbb{E}_{\bfz \sim \mathbb{S}^d}[d\left\langle \nabla F(\bfw), \bfz \right\rangle \bfz] \\
                                                                                         & = d\mathbb{E}_{\bfz \sim \mathbb{S}^d}[\bfz\bfz^\intercal] \nabla F(\bfw)                 \\
                                                                                         & = \nabla F(\bfw)
    \end{align*}
    where the last equality comes from Lemma \ref{lemma:exp-matrix}
\end{proof}
\begin{lemma}
    \label{lemma:norm-mu-zero}
    Let $\bfw \in \mathbb{R}^d$ and $\bfz_r$ independently sampled from $\mathbb{S}^d$, $r\in[k]$, %
    \begin{align*}
        \mathbb{E}_{\bfz_r \sim \mathbb{S}^d}\left[\left\|\frac{1}{k} \sum_{r=1}^k \gpop(\bfw, \bfz_r\zeroextraarg)\right\|_2^2\right] = \frac{d+k-1}{k} \left\|\nabla F(\bfw) \right\|_2^2.
    \end{align*}
\end{lemma}
\begin{proof}
    \begin{align*}
        \mathbb{E}_{\bfz_r \sim \mathbb{S}^d}\left[\left\|\frac{1}{k} \sum_{r=1}^k \gpop(\bfw, \bfz_r\zeroextraarg)\right\|_2^2\right] & =\mathbb{E}_{\bfz_r \sim \mathbb{S}^d}\left[\left\|\frac{d}{k} \sum_{r=1}^k \nabla F(\bfw)^\intercal \bfz_r \bfz_r \right\|_2^2\right]                                                      \\
                                                                                                                                                & = \frac{d^2}{k^2} \sum_{r=1}^k \sum_{s=1}^k\nabla F(\bfw)^\intercal  \mathbb{E}_{\bfz_r,\bfz_s  \sim \mathbb{S}^d}\left[\bfz_r\bfz_r^\intercal\bfz_s\bfz_s^\intercal\right] \nabla F(\bfw) \\
                                                                                                                                                & = \frac{d^2}{k^2}\Big(\frac{k}{d} + \frac{k(k-1)}{d^2}\Big)\|\nabla F(\bfw) \|_2^2
    \end{align*}
    where the last equality comes from \ref{lemma:cross-exp}. %
\end{proof}
\begin{lemma}
    \label{lemma:co-coercivity}
    (Co-coercivity of strongly convex functions, Lemma 3.11 \cite{bubeck2015convex}) Let F be $L_F$-smooth and $\lambda$-strongly convex. For all $\bfw_1,\bfw_2 \in \mathbb{R}^d$
    \begin{align*}
        (\nabla F(\bfw_1) - \nabla F(\bfw_2))^\intercal(\bfw_1 - \bfw_2) \ge \frac{\lambda L_F}{\lambda + L_F}\|\bfw_1 - \bfw_2\|_2^2 + \frac{1}{\lambda + L_F}\|\nabla F(\bfw_1) - \nabla F(\bfw_2)\|_2^2.
    \end{align*}
\end{lemma}

With these lemmas in hand, we can proceed to prove Theorem \ref{thm:convergence-mu-zero}. Let's consider the $t$-th step of
Algorithm \ref{alg:theoretical}, then for $\Pi_{\mathcal{W}}$ denoting the euclidean projection over $\cW$, we have %
\begin{align*}
    \left\|\bfw^{t+1} - \bfw^*\right\|_2 & = \left\|\Pi_{\mathcal{W}}\Big(\bfw^{t} - \eta \frac{1}{k} \sum_{r=1}^k \hat{g}_{\beta}(\bfw, \bfz_r^t\zeroextraarg)\bfz_r^t \Big)- \bfw^* \right\|_2                                                                                                                                 \\
                                          & \numrel{\le}  \left\|\bfw^{t} - \eta \frac{1}{k} \sum_{r=1}^k \hat{g}_{\beta}(\bfw, \bfz_r^t\zeroextraarg)\bfz_r^t - \bfw^* \right\|_2                                                                                                                                                  \\
                                          & \le \left\|\bfw^{t} - \eta \frac{1}{k} \sum_{r=1}^k \gpop(\bfw, \bfz_r^t\zeroextraarg) - \bfw^* \right\|_2 + \eta \left\| \frac{1}{k}\sum_{r=1}^k \hat{g}_{\beta}(\bfw^l, \bfz_r\zeroextraarg)\bfz_r  - \frac{1}{k}\sum_{r=1}^k \gpop(\bfw^l, \bfz_r\zeroextraarg)\right\|_2 \\
                                          & \numrel{\le} \left\|\bfw^{t} - \eta \frac{1}{k} \sum_{r=1}^k \gpop(\bfw, \bfz_r^t\zeroextraarg) - \bfw^* \right\|_2 + \eta\Delta
\end{align*}
\restartnumrel
where $(i)$ comes from the properties of Euclidean projection and $(ii)$ comes from conditioning on the event from Theorem \ref{thm:delta}.

Now, by taking the expectation over $Z^t$,
\begin{align*}
    \mathbb{E}_{Z^t}[\|\bfw^{t+1} - \bfw^*\|_2] \le
    \underbrace{\mathbb{E}_{Z^t}\left[\left\|\bfw^{t} - \eta \frac{1}{k} \sum_{r=1}^k \gpop(\bfw, \bfz_r^t\zeroextraarg) - \bfw^*\right\|_2\right]}_{T_1} + \eta \Delta
\end{align*}

We need to further bound the term $T_1$. We do that by considering
\begin{align*}
    &\mathbb{E}_{Z^t}\Bigg[\Bigg\|\bfw^{t} - \eta \frac{1}{k} \sum_{r=1}^k \gpop(\bfw, \bfz_r^t\zeroextraarg) - \bfw^*\Bigg\|_2^2\Bigg]                                                                                                                                                                                                \\
                                                & =  \mathbb{E}_{Z^t}\left[ \| \bfw^{t} - \bfw^*\|_2^2 - 2 \eta\left\langle \bfw^{t} - \bfw^* , \frac{1}{k} \sum_{r=1}^k \gpop(\bfw, \bfz_r^t\zeroextraarg)\right\rangle + \eta^2 \left\|\frac{1}{k} \sum_{r=1}^k \gpop(\bfw, \bfz_r^t\zeroextraarg) \right\|_2^2 \right] \\
                                                & \numrel{=} \| \bfw^{t} - \bfw^*\|_2^2 - 2 \eta\left\langle \bfw^{t} - \bfw^* , \nabla F (\bfw^t) \right\rangle + \eta^2 \frac{d+k-1}{k} \|\nabla F (\bfw^t) \|_2^2                                                                                                          \\
                                                & \numrel{\le} \left(1-\frac{2\lambda k}{(d+k-1)(L_F + \lambda)}\right)\| \bfw^{t} - \bfw^*\|_2^2 -\left(\frac{2}{L_F(\lambda + L_F)}-\frac{1}{L_F^2}\right)\frac{k}{d+k-1} \|\nabla F (\bfw^t) \|_2^2                                                                                     \\
                                                & \numrel{\le} \left(1-\frac{2\lambda k}{(d+k-1)(L_F + \lambda)}\right)\| \bfw^{t} - \bfw^*\|_2^2
\end{align*}
where $(i)$ comes from Lemmas \ref{lemma:expected-mu-zero} and \ref{lemma:norm-mu-zero}, $(ii)$ comes from Lemma \ref{lemma:co-coercivity}, Assumption \ref{ass:local-minimum} and the choice of $\eta = \frac{k}{L_F(d+k-1)}$ and $(iii)$ comes from the fact that $\lambda \le L_F$.

Then by applying Jensen's inequality on the square root function and using the fact that $\sqrt{1-x} \le 1- \frac{x}{2}$, we can bound $T_1$ as
\begin{align*}
    \mathbb{E}_{Z^t}\Big[\Big\|\bfw^{t} - \eta \frac{1}{k} \sum_{r=1}^k \gpop(\bfw, \bfz_r^t\zeroextraarg) - \bfw^*\Big\|_2\Big] \le \Big(1-\frac{\lambda k}{(d+k-1)(L_F + \lambda)}\Big)\| \bfw^{t} - \bfw^*\|_2
\end{align*}

Thus, %
\begin{align*}
    \mathbb{E}_{Z^t}[\|\bfw^{t+1} - \bfw^*\|_2] \le \Big(1-\frac{\lambda k}{(d+k-1)(L_F + \lambda)}\Big)\| \bfw^{t} - \bfw^*\|_2 + \frac{k}{(d+k-1)L_F}\Delta
\end{align*}

The proof is finished by taking the expectation over $\mathcal{Z}^{t-1}$ and recursively applying the inequality. %
\subsection{Proof of Theorem \ref{thm:convergence-mu-positive}}
\label{proof:thm-convergence-mu-positive}
For Theorem \ref{thm:convergence-mu-positive} we need similar lemmas as we had for the proof of Theorem \ref{thm:convergence-mu-zero}, in particular, we need the expected value of the zero-order estimate and a bound on its squared norm.
\restartnumrel
\begin{lemma}
    \label{lemma:smooth-f}
    (Lemma 4.1 $(a)$ \cite{gao2018information}) Let $\bfw \in \mathbb{R}^d$ and let $F_\mu$ be as defined in Definition \ref{def:smoothed-version}, then
    \begin{align*}
        \quad \nabla F_\mu (\bfw) = \mathbb{E}_{\bfz \sim \mathbb{S}^d}\Big[\dfrac{d}{\mu}F(\bfw + \mu \bfz)\bfz\Big]. \\
    \end{align*}
\end{lemma}
\begin{lemma}
    \label{lemma:expected-mu-positive}
    Let $\bfw \in \mathbb{R}^d$ and let $F_\mu$ be as defined in Definition \ref{def:smoothed-version}, then
    \begin{align*}
        \mathbb{E}_{\bfz \sim \mathbb{S}^d}[\gpop(\bfw, \bfz\muextraarg)] = \nabla F_\mu (\bfw).
    \end{align*}

\end{lemma}
\begin{proof}
    \begin{align*}
        \mathbb{E}_{\bfz \sim \mathbb{S}^d}[\gpop(\bfw, \bfz\muextraarg)] & = \frac{1}{2}\mathbb{E}_{\bfz \sim \mathbb{S}^d}\Big[\frac{d}{\mu}(F(\bfw + \mu \bfz) - F(\bfw - \mu \bfz)) \bfz\Big]                                                                                  \\
                                                                                           & = \frac{1}{2}\mathbb{E}_{\bfz \sim \mathbb{S}^d}\Big[\frac{d}{\mu}F(\bfw + \mu \bfz) \bfz\Big] - \frac{1}{2}\mathbb{E}_{\bfz \sim \mathbb{S}^d}\Big[\frac{d}{\mu}F(\bfw - \mu \bfz)\bfz\Big] \\
                                                                                           & = \mathbb{E}_{\bfz \sim \mathbb{S}^d}\Big[\dfrac{d}{\mu}F(\bfw + \mu \bfz)\bfz\Big]
    \end{align*}
    where the last equality comes from $\bfz$ and $-\bfz$ being identically distributed due to the symmetry of $\mathbb{S}^d$.
\end{proof}
\begin{lemma}
    (Equation (12) \cite{nesterovRandomGradientFreeMinimization2017}) Let $F_\mu$ be as defined in Definition \ref{def:smoothed-version}. For any $\mu > 0$, if $F$ is $L_F$-smooth, then $F_\mu$ is $L_F$-smooth.
\end{lemma}
\begin{lemma}
    \label{lemma:norm-f-norm-smooth}
    Let $\bfw \in \mathbb{R}^d$ and let $F_\mu$ be as defined in Definition \ref{def:smoothed-version}, then
    \begin{align*}
        \|\nabla F(\bfw)\|_2^2 \le 2\|\nabla F_\mu (\bfw)\|_2^2 + \frac{L_F^2\mu^2d^2}{2}.
    \end{align*}
\end{lemma}
\begin{proof}
    From Lemma 4.1 $(b)$ in \cite{gao2018information}, we know that $ \|\nabla F(\bfw) - \nabla F_\mu (\bfw)\|_2 \le \frac{L_F\mu d}{2}$. The result follows from applying the triangle inequality and the fact that $(x+y)^2 \le \frac{x^2 + y^2}{2}$.
\end{proof}
\begin{lemma}
    Let $F_\mu$ be as defined in Definition \ref{def:smoothed-version}. For any $\mu > 0$, if $F$ is $\lambda$-strongly convex, then $F_\mu$ is $\lambda$-strongly convex. %
\end{lemma}
\begin{proof}
    Let $\bfw_1, \bfw_2 \in \mathbb{R}^d$, then %
    \begin{align*}
        & F_\mu(\bfw_1) - F_\mu(\bfw_2) - \left\langle \nabla F_\mu (\bfw_2),  \bfw_1 - \bfw_2 \right\rangle \\
        & = \frac{1}{A_d} \int_{\mathbb{S}^d} [F(\bfw_1 + \mu \bfz) - F(\bfw_2 + \mu \bfz) - \left\langle \nabla F(\bfw_2+ \mu \bfz),  \bfw_1 - \bfw_2 \right\rangle ]d \bfz \\
                                                                                                                            & \ge  \frac{1}{A_d} \int_{\mathbb{S}^d} \frac{\lambda}{2} \|\bfw_1 - \bfw_2\|_2^2                                                                                                                \\
                                                                                                                            & = \frac{\lambda}{2} \|\bfw_1 - \bfw_2\|_2^2,
    \end{align*}
    where $A_d$ is the surface area of the unit sphere.
\end{proof}

\begin{lemma}
    \label{lemma:norm-mu-positive}
    Let $\bfw \in \mathbb{R}^d$,
    \begin{align*}
        \mathbb{E}_{\bfz \sim \mathbb{S}^d}[\|\gpop(\bfw, \bfz\muextraarg)\|_2^2] \le 2d \|\nabla F(\bfw)\|_2^2 + \frac{L_F^2\mu^2d^2}{2}.
    \end{align*}
\end{lemma}
\begin{proof}
    Let $A_d$ be the surface area of the unit sphere $\mathbb{S}^d$, then
    \begin{align*}
        &\mathbb{E}_{\bfz \sim \mathbb{S}^d}[\|\gpop(\bfw, \bfz\muextraarg)\|_2^2] = \frac{1}{A_d} \int_{\mathbb{S}^d} \frac{d^2}{4\mu^2} (F(\bfw + \mu \bfz) - F(\bfw - \mu \bfz))^2 \|\bfz\|_2^2 d \bfz                                                                                                        \\
                                                                                                   & = \frac{d^2}{4A_d\mu^2} \int_{\mathbb{S}^d}(F(\bfw + \mu \bfz) - F(\bfw - \mu \bfz))^2d \bfz                                                                                                                                        \\
                                                                                                   & \numrel{\le} \frac{d^2}{2A_d\mu^2}\Big[ \int_{\mathbb{S}^d}(F(\bfw + \mu \bfz) - F(\bfw))^2d \bfz  + \int_{\mathbb{S}^d}(F(\bfw - \mu \bfz) - F(\bfw))^2d \bfz   \Big]                                            \\
                                                                                                   & \numrel{\le}\frac{d^2}{A_d\mu^2}\Big[\int_{\mathbb{S}^d}(F(\bfw + \mu \bfz) - F(\bfw) - \left\langle \nabla F(\bfw), \mu \bfz\right\rangle)^2d \bfz                                                                                      \\
                                                                                                   & \quad \quad \quad \quad + \int_{\mathbb{S}^d}(F(\bfw - \mu \bfz) - F(\bfw) + \left\langle \nabla F(\bfw), \mu \bfz\right\rangle)^2d \bfz + 2 \int_{\mathbb{S}^d} (\left\langle \nabla F(\bfw), \mu \bfz\right\rangle)^2d\bfz\Big] \\
                                                                                                   & \numrel{\le} \frac{d^2}{A_d\mu^2} \Big[2A_d\frac{L_F^2\mu^4}{4} + 2A_d\frac{\mu^2}{d} \|\nabla F(\bfw)\|_2^2\Big]
    \end{align*}
    where $(i)$ and $(ii)$ use the fact that $(x+y)^2 \le 2x^2 + 2y^2$ and $(iii)$ uses both the $L_F$-smoothness of $F$ and Lemma \ref{lemma:exp-matrix}.
\end{proof}
\begin{lemma}
    \label{lemma:very-hard-to-prove}
    Let $\mathbf{x} \in \mathbb{R}^d$ and $\bfz_1, \bfz_2$ be independently sampled from $\mathbb{S}^d$, then
    \begin{align*}
        \mathbb{E}_{\bfz_1,\bfz_2  \sim \mathbb{S}^d}[|\bfz_1^t\bfz_2|(\mathbf{x}^\intercal\bfz_1)^2] \le \dfrac{\|x\|_2^2}{\sqrt{d^3}}
    \end{align*}
\end{lemma}
\begin{proof}
    Consider
    \begin{align*}
        \mathbb{E}_{\bfz_1,\bfz_2  \sim \mathbb{S}^d}[|\bfz_1^t\bfz_2|(\mathbf{x}^\intercal\bfz_1)^2] & = \mathbb{E}_{\bfz_1  \sim \mathbb{S}^d}[\mathbb{E}_{\bfz_2  \sim \mathbb{S}^d}[|\bfz_1^\intercal\bfz_2|](\mathbf{x}^\intercal\bfz_1)^2] \\
    \end{align*}

    Consider now the rotation matrix $R\in\mathbb{R}^{d\times d}$ that rotates $\bfz_1$ to $\mathbf{e}_1$. Due to the symmetry of the unit sphere $\mathbb{S}^d$, the two random variables $\bfz_2$ and $R\bfz_2$ are identically distributed. Then, since $\left\langle\bfz_1, \bfz_2\right\rangle = \left\langle R\bfz_1, R\bfz_2\right\rangle$,
    \begin{align*}
        \mathbb{E}_{\bfz_1,\bfz_2  \sim \mathbb{S}^d}[(|\bfz_1^t\bfz_2|(\mathbf{x}^\intercal\bfz_1)^2)^2] & = \mathbb{E}_{\bfz_2  \sim \mathbb{S}^d}[|\mathbf{e}_1^\intercal\bfz_2|] \mathbb{E}_{\bfz_1 \sim \mathbb{S}^d}[(\mathbf{x}^\intercal\bfz_1)^2]             \\
                                                                                                                                & = \|x\|_2^2 \mathbb{E}_{\bfz_2  \sim \mathbb{S}^d}[|\mathbf{e}_1^\intercal\bfz_2|] \mathbb{E}_{\bfz_1 \sim \mathbb{S}^d}[(\mathbf{e}_1^\intercal\bfz_1)^2]
    \end{align*}
    by also applying the same rotation argument to $\mathbf{x}$.

    Now, since for any $z$ sampled from the unit sphere $\mathbb{S}^d$, we have
    \begin{align*}
        1 & = \mathbb{E}_{\bfz \sim \mathbb{S}^d}[\|\bfz\|^2] \\ & = \mathbb{E}_{\bfz \sim \mathbb{S}^d}\Big[\sum_{i=1}^d (\bfz^{(i)})^2\Big] \\ & = \sum_{i=1}^d\mathbb{E}_{\bfz \sim \mathbb{S}^d}[ (\bfz^{(i)})^2].
    \end{align*}
    By the symmetry of the
    unit sphere, $\mathbb{E}_{\bfz \sim \mathbb{S}^d}[(\bfz^{(i)})^2 ] = \frac{1}{d}$ for all $i\in [d]$, where $\bfz^{(i)}$ denotes the $i$-th coordinate of $\bfz$.
    Lastly, since
    \begin{align*}
        \mathbb{E}_{\bfz \sim \mathbb{S}^d}[(\mathbf{e}_1^\intercal\bfz)^2] & = \mathbb{E}_{\bfz \sim \mathbb{S}^d}[(\bfz^{(1)})^2 ] \\ &= \frac{1}{d}, \\
    \end{align*}
    and, by applying Jensen's inequality on the square root function,
    \begin{align*}
        \mathbb{E}_{\bfz \sim \mathbb{S}^d}[|\mathbf{e}_1^\intercal\bfz|] & = \mathbb{E}_{\bfz \sim \mathbb{S}^d}\Big[\sqrt{(\mathbf{e}_1^\intercal\bfz)^2}\Big] \\
                                                                              & \le  \frac{1}{\sqrt{d}},                                                                 \\
    \end{align*}
    we obtain the statement in the lemma.
\end{proof}
\restartnumrel
\begin{lemma}
    \label{lemma:prod_dist}
    Let $a,b,c,d,\delta \in \mathbb{R}$, $\delta > 0$, such that
    \begin{align*}
        |a-c| \le \delta, \quad |b-d| \le \delta,
    \end{align*}
    then
    \begin{align*}
        |ab-cd| \le \frac{c^2 + d^2}{2} + 2\delta^2.
    \end{align*}
\end{lemma}
\begin{proof}
    \begin{align*}
        |ab-cd| = |ab-cb+cb-cd| & \le |a-c||b| + |c||b-d|             \\
                                & \le \delta(|d|+\delta) + |c|\delta  \\
                                & = |c|\delta +|d|\delta + \delta^2   \\
                                & \le \frac{c^2 + d^2}{2} + 2\delta^2,
    \end{align*}
    where the last inequality uses the fact that $xy \le \frac{x^2 + y^2}{2}$.
\end{proof}

\begin{lemma}
    \label{lemma:norm-cross-mu-positive}
    Let $\bfz_1, \bfz_2$ be independently sampled from $\mathbb{S}^d$. Then
    \begin{align*}
        \mathbb{E}_{\bfz_1,\bfz_2  \sim \mathbb{S}^d}[\left\langle \gpop(\bfw, \bfz_1\muextraarg), \gpop(\bfw, \bfz_2\muextraarg)\right\rangle] \le (1+\sqrt{d})\|\nabla F(\bfw)\|_2^2 + 2L_F^2\mu^2d^2.
    \end{align*}
\end{lemma}
\restartnumrel
\begin{proof}
    Let $A_d$ be the surface area of the unit sphere $\mathbb{S}^d$ and define $\mathbb{S}^d_+(\bfz) = \left\{\mathbf{x}: \mathbf{x}\in \mathbb{S}^d, \bfz^\intercal\mathbf{x} > 0\right\}$ and similarly $\mathbb{S}^d_-(\bfz) = \left\{\mathbf{x}: \mathbf{x}\in \mathbb{S}^d, \bfz^\intercal\mathbf{x} < 0\right\}$, then %
    \begin{align*}
        &\mathbb{E}_{\bfz_1,\bfz_2  \sim \mathbb{S}^d}  [\left\langle \gpop(\bfw, \bfz_1\muextraarg), \gpop(\bfw, \bfz_2\muextraarg)\right\rangle]                                                                                                                                                                                                                                       \\
                                                                  & = \frac{d^2}{A_d^2}\Big[\int_{\mathbb{S}^d} \int_{\mathbb{S}^d_+(\bfz_1)} \frac{F(\bfw + \mu \bfz_1) - F(\bfw - \mu \bfz_1)}{2\mu} \frac{F(\bfw + \mu \bfz_2) - F(\bfw - \mu \bfz_2)}{2\mu} \bfz_1^\intercal \bfz_2 d\bfz_2 d\bfz_1                                       \\
                                                                  & \quad \quad +\int_{\mathbb{S}^d} \int_{\mathbb{S}^d_-(\bfz_1)} \frac{F(\bfw + \mu \bfz_1) - F(\bfw - \mu \bfz_1)}{2\mu} \frac{F(\bfw + \mu \bfz_2) - F(\bfw - \mu \bfz_2)}{2\mu} \bfz_1^\intercal \bfz_2 d\bfz_2 d\bfz_1 \Big]                                            \\
                                                                  & \numrel{\le} \frac{d^2}{A_d^2}\Big[\int_{\mathbb{S}^d} \int_{\mathbb{S}^d_+(\bfz_1)} \Big( \nabla F(\bfw)^\intercal\bfz_1\nabla F(\bfw)^\intercal\bfz_2 + \frac{(\nabla F(\bfw)^\intercal\bfz_1)^2 + (\nabla F(\bfw)^\intercal\bfz_1)^2}{2} + 2L_F^2\mu^2\Big)\bfz_1^\intercal \bfz_2 d\bfz_2 d\bfz_1 \\
                                                                  & \quad \quad + \int_{\mathbb{S}^d} \int_{\mathbb{S}^d_-(\bfz_1)} \Big( \nabla F(\bfw)^\intercal\bfz_1\nabla F(\bfw)^\intercal\bfz_2 - \frac{(\nabla F(\bfw)^\intercal\bfz_1)^2 + (\nabla F(\bfw)^\intercal\bfz_1)^2}{2} - 2L_F^2\mu^2\Big)\bfz_1^\intercal \bfz_2 d\bfz_2 d\bfz_1                      \\
                                                                  & = \mathbb{E}_{\bfz_1,\bfz_2  \sim \mathbb{S}^d} \Bigg[ \nabla F(\bfw)^\intercal\bfz_1\bfz_1^\intercal\bfz_2\bfz_2^\intercal\nabla F(\bfw) + \Bigg(\frac{(\nabla F(\bfw)^\intercal\bfz_1)^2 + (\nabla F(\bfw)^\intercal\bfz_2)^2}{2} + 2L_F^2\mu^2\Bigg) |\bfz_1^\intercal\bfz_2|\Bigg]                  \\
                                                                  & \numrel{\le} d^2 (\nabla F(\bfw)^\intercal\mathbb{E}_{\bfz_1,\bfz_2  \sim \mathbb{S}^d}[\bfz_1\bfz_1^\intercal\bfz_2\bfz_2^\intercal]\nabla F(\bfw) +\mathbb{E}_{\bfz_1,\bfz_2  \sim \mathbb{S}^d}[|\bfz_1^\intercal\bfz_2|(\nabla F(\bfw)^\intercal\bfz_1)^2 + 2L_F^2\mu^2])                   \\
                                                                  & \numrel{\le} d^2 \Bigg(\frac{1}{d^2} + \frac{1}{\sqrt{d^3}}\Bigg)\|\nabla F(\bfw)\|_2^2 + 2L_F^2\mu^2d^2
    \end{align*}
    where $(i)$ uses Lemma \ref{lemma:prod_dist} in conjunction with the $L_f$ smootheness property of $F$ that $|\frac{F(\bfw + \mu \bfz) - F(\bfw - \mu \bfz)}{2\mu} - \left\langle \nabla F(\bfw), \bfz\right\rangle| \le L_F\mu$ for any $\bfz \in \mathbb{S}^d$, $(ii)$ uses the fact that $|\bfz_1^\intercal\bfz_2| \le 1$, and $(iii)$ uses Lemmas \ref{lemma:cross-exp} and \ref{lemma:very-hard-to-prove}. %
\end{proof}
\begin{lemma}
    \label{lemma:full-normal-mu-positive}
    Let $\bfw \in \mathbb{R}^d$ and $\bfz_r$ independently sampled from $\mathbb{S}^d$, $r\in[k]$,
    \begin{align*}
        \mathbb{E}_{\bfz_r \sim \mathbb{S}^d}\Bigg[\Big\|\frac{1}{k} \sum_{r=1}^k \gpop(\bfw, \bfz_r\muextraarg)\Big\|_2^2\Bigg] = \frac{(k-1)(1+\sqrt{d}) +2d }{k} \|\nabla F(\bfw) \|_2^2 + 2L_F^2\mu^2d^2.
    \end{align*}
\end{lemma}
\restartnumrel
\begin{proof}
    \begin{align*}
        &\mathbb{E}_{\bfz_r \sim \mathbb{S}^d}\Bigg[\Big\|\frac{1}{k} \sum_{r=1}^k \gpop(\bfw, \bfz_r\muextraarg)\Big\|_2^2\Bigg]  =  \mathbb{E}_{\bfz_r \sim \mathbb{S}^d}\Bigg[\frac{1}{k^2}\sum_{r=1}^k \sum_{s=1}^k \gpop(\bfw, \bfz_r\muextraarg)^\intercal\gpop(\bfw, \bfz_s\muextraarg)\Bigg] \\
                                                                                                                                                  & \numrel{\le} \frac{1}{k^2}[k(2d \|\nabla F(\bfw) \|_2^2 +2L_f^2\mu^2d^2) + k(k-1)((1+\sqrt{d})\|\nabla F(\bfw) \|_2^2 + 2L_f^2\mu^2d^2)]                                  \\
                                                                                                                                                  & =\frac{(k-1)(1+\sqrt{d}) +2d }{k} \|\nabla F(\bfw) \|_2^2 + 2L_F^2\mu^2d^2,
    \end{align*}
    where $(i)$ uses Lemmas \ref{lemma:norm-mu-positive} and \ref{lemma:norm-cross-mu-positive}
\end{proof}

Now we can proceed with the proof of Theorem \ref{thm:convergence-mu-positive}.

Let's consider the $t$-th step of
Algorithm \ref{alg:theoretical}, then, by repeating the same idea in the proof of Theorem \ref{thm:convergence-mu-zero}, we have

\begin{align*}
    \mathbb{E}_{Z^t}[\|\bfw^{t+1} - \bfw_\mu^*\|_2] \le
    \underbrace{\mathbb{E}_{Z^t}\Big[\Big\|\bfw^{t} - \eta \frac{1}{k} \sum_{r=1}^k \gpop(\bfw, \bfz_r^t\muextraarg) - \bfw_\mu^*\Big\|_2\Big]}_{T_1} + \eta \Delta
\end{align*}
\restartnumrel
The fundamental difference lies in how we bound the term $T_1$. Let $\tau = \frac{2d+ (k-1)\big(1+\sqrt{d}\big)}{k}$ %
\begin{align*}
    &\mathbb{E}_{Z^t}\Big[\Big\|\bfw^{t} -  \eta \frac{1}{k} \sum_{r=1}^k \gpop(\bfw, \bfz_r^t\muextraarg) - \bfw_\mu^*\Big\|_2^2\Big]                                                                                                                                                                                                      \\
                                                & =  \mathbb{E}_{Z^t}\Big[ \| \bfw^{t} - \bfw_\mu^*\|_2^2 - 2 \eta\left\langle \bfw^{t} - \bfw_\mu^* , \frac{1}{k} \sum_{r=1}^k \gpop(\bfw, \bfz_r^t\muextraarg)\right\rangle + \eta^2 \Big\|\frac{1}{k} \sum_{r=1}^k \gpop(\bfw, \bfz_r^t\muextraarg) \Big\|_2^2 \Big] \\
                                                & \numrel{=} \| \bfw^{t} - \bfw_\mu^*\|_2^2 - 2 \eta\left\langle \bfw^{t} - \bfw_\mu^* , \nabla F_\mu (\bfw^t) \right\rangle + \eta^2 \tau \|\nabla F (\bfw^t) \|_2^2    + 2\eta^2   L_F^2\mu^2d^2                                                                                        \\
                                                & \numrel{\le} \Bigg(1-\frac{2\eta \lambda L_F}{L_F + \lambda}\Bigg)\| \bfw^{t} - \bfw_\mu^*\|_2^2 -\Bigg(\frac{2\eta}{L_F + \lambda} - 2\eta^2\tau\Bigg) \|\nabla F_\mu (\bfw^t) \|_2^2     + 2\eta^2\Big(1+\frac{\tau}{4}\Big)   L_F^2\mu^2d^2                                                 \\
                                                & \numrel{=} \Bigg(1-\frac{\lambda}{(L_F + \lambda)\tau}\Bigg)\| \bfw^{t} - \bfw_\mu^*\|_2^2 -\Bigg(\frac{1}{L_F(\lambda + L_F)\tau }-\frac{1}{2L_F^2\tau}\Bigg)\|\nabla F_\mu (\bfw^t) \|_2^2      + \frac{1}{2L_F^2\tau^2}\Big(1+\frac{\tau}{4}\Big)   L_F^2\mu^2d^2                           \\
                                                & \numrel{\le} \Bigg(1-\frac{\lambda}{(L_F + \lambda)\tau}\Bigg)\| \bfw^{t} - \bfw_\mu^*\|_2^2  + \frac{1}{2\tau^2}\Big(1+\frac{\tau}{4}\Big) \mu^2d^2
\end{align*}
where $(i)$ comes from Lemmas \ref{lemma:expected-mu-positive} and \ref{lemma:full-normal-mu-positive}, $(ii)$ comes from Lemmas \ref{lemma:co-coercivity}, \ref{lemma:norm-f-norm-smooth} and Assumption \ref{ass:local-minimum-smoothed}, $(iii)$ comes from choosing $\eta = \frac{1}{2\tau L_F}$ and $(iv)$ comes from the fact that $\lambda \le L_F$.

Then by aplying Jensen's inequality on the square root function and using the fact that $\sqrt{x+y} \le \sqrt{x} + \sqrt{y}$ and $\sqrt{1-x} \le 1- \frac{x}{2}$, we can bound $T_1$ as
\begin{align*}
    \mathbb{E}_{Z^t}\Big[\Big\|\bfw^{t} - \eta \frac{1}{k} \sum_{r=1}^k \gpop(\bfw, \bfz_r^t\zeroextraarg) - \bfw_\mu^*\Big\|_2\Big] \le \Bigg(1-\frac{\lambda}{2(L_F + \lambda)\tau}\Bigg)\| \bfw^{t} - \bfw_\mu^*\|_2 + \frac{\mu d}{\sqrt{2}\tau}\sqrt{1+\frac{\tau}{4}}.
\end{align*}

Thus,
\begin{align*}
    \mathbb{E}_{Z^t}[\|\bfw^{t+1} - \bfw_\mu^*\|_2] \le \Bigg(1-\frac{\lambda}{2(L_F + \lambda)\tau}\Bigg)\| \bfw^{t} - \bfw_\mu^*\|_2 + \frac{1}{2L_F\tau}\Delta + \frac{\mu d}{\sqrt{2}\tau}\sqrt{1+\frac{\tau}{4}}.
\end{align*}

We conclude the proof by taking the expectation over $\mathcal{Z}^{t-1}$ and recursively applying the inequality. $ \square $ %

\end{document}